\documentclass[a4paper,11pt]{article}
\usepackage[top=4cm, bottom=4cm, left=3cm, right=3cm]{geometry}

\usepackage{amsmath,amssymb,amsthm}
\usepackage{graphicx,color}
\usepackage[hyphens]{url}
\usepackage{dsfont}
\usepackage{booktabs}
\usepackage[round,sort]{natbib}
\usepackage[sort,nocompress,noadjust]{cite}
\usepackage[normalem]{ulem}
\usepackage[boxed]{algorithm}
\usepackage{algpseudocode}
\usepackage{mathtools}
\usepackage[nameinlink,capitalize]{cleveref}

\newtheorem{theorem}{Theorem}
\newtheorem{cor}{Corollary}
\newtheorem{prop}{Proposition}
\newtheorem{lemma}{Lemma}

\newtheorem{applemma}{Lemma}

\newtheorem{appcor}{Corollary}

\theoremstyle{definition}
\newtheorem{defin}{Definition}
\newtheorem{assumption}{Assumption}

\newcommand{\cM}{\mathcal{M}}

\newcommand{\cS}{\mathcal{S}}

\newcommand{\cX}{\mathcal{X}}

\newcommand{\NN}{\mathbb{N}}

\newcommand{\PP}{\mathbb{P}}

\newcommand{\RR}{\mathbb{R}}

\newcommand{\bone}{\mathbf{1}}

\newcommand*{\ep}{\varepsilon}
\newcommand*{\eps}{\varepsilon}
\newcommand*{\defeq}{:=}

\DeclareMathOperator*{\argmin}{argmin}

\renewcommand{\leq}{\leqslant}
\renewcommand{\geq}{\geqslant}

\newcommand*{\Phat}{\hat{P}}

\newcommand*{\Peta}{P^\eta}

\newcommand*{\Ctilde}{\tilde{C}}
\newcommand*{\Ktilde}{\tilde{K}}
\newcommand*{\Obj}{W_\eta(\p,\q)}
\newcommand*{\Coup}{\cM(\p, \q)}

\newcommand*{\Couptilde}{\cM(\tilde \p, \tilde \q)}
\newcommand*{\ent}{H}
\newcommand*{\p}{\mathbf{p}}
\newcommand*{\q}{\mathbf{q}}
\newcommand*{\Sinkcoup}{\Pi^{\cS}_{\Coup}}
\newcommand*{\Sink}{\Pi^{\cS}}
\newcommand*{\SinkcM}{\Pi_{\cM}^{\cS}}
\newcommand*{\Sinkcouptilde}{\Pi^{\cS}_{\Couptilde}}
\newcommand{\KL}[2]{\mathsf{KL}(#1 \| #2)}

\providecommand{\abs}[1]{\left\lvert#1\right\rvert}
\newcommand*{\Otilde}{\tilde{O}}

\newcommand*{\VC}{V_C}
\newcommand*{\VCtilde}{V_{\Ctilde}}

\newcommand*{\Cinf}{\|C\|_{\infty}}
\newcommand*{\Ctildeinf}{\|\tilde{C}\|_{\infty}}

\DeclareMathOperator{\half}{\frac{1}{2}}

\providecommand{\abs}[1]{\lvert{#1}\rvert}

\providecommand{\tr}{\operatorname{Tr}}

\newcommand{\R}{\mathbb R}

\newcommand{\N}{\mathbb N}

\newcommand{\hh}{\mathcal H}

\newcommand{\la}{\lambda}
\newcommand{\deff}{d_{\textrm{eff}}}

\newcommand{\eqals}[1]{\begin{align*}#1\end{align*}}
\newcommand{\eqal}[1]{\begin{align}#1\end{align}}

\newcommand{\rank}{\text{rank}}
\newcommand{\op}{\text{op}}

\newcommand{\Kt}{\widetilde{K}}

\newcommand*{\Round}{\textsc{Round}}
\newcommand*{\Sinkhorn}{\textsc{Sinkhorn}}
\newcommand*{\AdaptiveNsytrom}{\textsc{AdaptiveNystr\"om}}
\newcommand*{\Nystrom}{\textsc{Nystr\"om}}

\newcommand*{\Tmult}{\textsc{T}_{\textsc{mult}}}

\newcommand{\diag}{\mathbb{D}}
\newcommand{\errr}{\text{err}_r}
\newcommand{\errc}{\text{err}_c}
\begin{document}

\title{Massively scalable Sinkhorn distances via the Nystr\"om method}
\newcommand*\samethanks[1][\value{footnote}]{\footnotemark[#1]}

\author{
        Jason Altschuler\footnote{Supported in part by NSF Graduate Research Fellowship 1122374.}\\
        MIT\\
        \texttt{jasonalt@mit.edu}
        \and
        Francis Bach\footnote{Supported in part by the European Research Council (grant SEQUOIA 724063).} \\
        INRIA - ENS - PSL \\
        \texttt{francis.bach@inria.fr}
        \and
        Alessandro Rudi\samethanks\\
        INRIA - ENS - PSL \\
        \texttt{alessandro.rudi@inria.fr}
        \and 
        Jonathan Niles-Weed\footnote{Supported in part by the Josephine de K\'arm\'an Fellowship.} \\
        NYU \\
        \texttt{jnw@cims.nyu.edu}
}
\date{}
\maketitle
\begin{abstract}
The Sinkhorn ``distance,'' a variant of the Wasserstein distance with entropic regularization, is an increasingly popular tool in machine learning and statistical inference. However, the time and memory requirements of standard algorithms for computing this distance grow quadratically with the size of the data, making them prohibitively expensive on massive data sets. In this work, we show that this challenge is surprisingly easy to circumvent: combining two simple techniques---the Nystr\"om method and Sinkhorn scaling---provably yields an accurate approximation of the Sinkhorn distance with significantly lower time and memory requirements than other approaches. We prove our results via new, explicit analyses of the Nystr\"om method and of the stability properties of Sinkhorn scaling. We validate our claims experimentally by showing that our approach easily computes Sinkhorn distances on data sets hundreds of times larger than can be handled by other techniques.
\end{abstract}

\newpage
\tableofcontents
\newpage

\section{Introduction}\label{sec:intro}

Optimal transport is a fundamental notion in probability theory and geometry~\citep{Vil08}, which has recently attracted a great deal of interest in the machine learning community as a tool for image recognition~\citep{LiWanZha13,RubTomGui00}, domain adaptation~\citep{CouFlaTui17,CouFlaTui14}, and generative modeling~\citep{BouGelTol17,ArjChiBot17,GenCutPey16}, among many other applications~\citep[see, e.g.,][]{PeyCut17,KolParTho17}.

The growth of this field has been fueled in part by computational advances, many of them stemming from an influential proposal of~\citet{Cut13} to modify the definition of optimal transport to include an \emph{entropic penalty}.
The resulting quantity, which~\citet{Cut13} called the \emph{Sinkhorn ``distance''}\footnote{We use quotations since it is not technically a distance; see~\citep[Section 3.2]{Cut13} for details. The quotes are dropped henceforth.}
after~\citet{Sin67}, is significantly faster to compute than its unregularized counterpart. Though originally attractive purely for computational reasons, the Sinkhorn distance has since become an object of study in its own right because it appears to possess better statistical properties than the unregularized distance both in theory and in practice~\citep{GenPeyCut18, MonMulCut16, PeyCut17, SchShuTab17, RigWee18b}.
Computing this distance as quickly as possible has therefore become an area of active study.

We briefly recall the setting.
Let $\p$ and $\q$ be probability distributions supported on at most $n$ points in~$\RR^d$.
We denote by $\Coup$ the set of all \emph{couplings} between $\p$ and $\q$, and for any $P \in \Coup$, we denote by $\ent(P)$ its Shannon entropy.
(See Section~\ref{sec:prelim} for full definitions.)
The Sinkhorn distance between $\p$ and $\q$ is defined as 
\begin{equation}\label{eq:peta}
\Obj
\defeq \min_{P \in \Coup} \sum_{ij} P_{ij} \|x_i - x_j\|_2^2 - \eta^{-1} \ent(P)\,,
\end{equation}
for a parameter $\eta > 0$.
We stress that we use the squared Euclidean cost in our formulation of the Sinkhorn distance. This choice of cost---which in the unregularized case corresponds to what is called the $2$-Wasserstein distance~\citep{Vil08}---is essential to our results, and we do not consider other costs here. The squared Euclidean cost is among the most common in applications~\citep{SchShuTab17,GenPeyCut18,CouFlaTui17,ForHutNit18,BouGelTol17}.

Many algorithms to compute $\Obj$ are known. \citet{Cut13} showed that a simple iterative procedure known as Sinkhorn's algorithm had very fast performance in practice, and later experimental work has shown that greedy and stochastic versions of Sinkhorn's algorithm perform even better in certain settings~\citep{GenCutPey16, AltWeeRig17}. These algorithms are notable for their versatility: they provably succeed for any bounded, nonnegative cost. On the other hand, these algorithms are based on matrix manipulations involving the $n \times n$ cost matrix $C$, so their running times and memory requirements inevitably scale with $n^2$. In experiments, \citet{Cut13} and \citet{GenCutPey16} showed that these algorithms could reliably be run on problems of size $n \approx 10^4$.

Another line of work has focused on obtaining better running times when the cost matrix has special structure. A preeminent example is due to~\citet{SolGoePey15}, who focus on the Wasserstein distance on a compact Riemannian manifold, and show that an approximation to the entropic regularized Wasserstein distance can be obtained by repeated convolution with the heat kernel on the domain.
\citet{SolGoePey15} also establish that for data supported on a grid in $\RR^d$, significant speedups are possible by decomposing the cost matrix into ``slices'' along each dimension~\citep[see][Remark 4.17]{PeyCut17}.
While this approach allowed Sinkhorn distances to be computed on significantly larger problems ($n \approx 10^8$), it does not extend to non-grid settings.
Other proposals include using random sampling of auxiliary points to approximate semi-discrete costs~\citep{TenWolKim18} or performing a Taylor expansion of the kernel matrix in the case of the squared Euclidean cost~\citep{AltBacRud18}.
These approximations both focus on the $\eta \to \infty$ regime, when the regularization term in~\eqref{eq:peta} is very small, and do not apply to the moderately regularized case $\eta = O(1)$ typically used in practice. Moreover, the running time of these algorithms scales exponentially in the ambient dimension, which can be very large in applications.

\subsection{Our contributions}
We show that a simple algorithm can be used to approximate $\Obj$ quickly on massive data sets.
Our algorithm uses only known tools, but we give novel theoretical guarantees that allow us to show that the Nystr\"om method combined with Sinkhorn scaling \emph{provably} yields a valid approximation algorithm for the Sinkhorn distance at a fraction of the running time of other approaches.


We establish two theoretical results of independent interest: \textbf{(i)} New Nystr\"om approximation results showing that instance-adaptive low-rank approximations to Gaussian kernel matrices can be found for data lying on a low-dimensional manifold (Section~\ref{sec:nystrom}). \textbf{(ii)} New stability results about Sinkhorn projections, establishing that a sufficiently good approximation to the cost matrix can be used (Section~\ref{sec:sink}).

\subsection{Prior work}
Computing the Sinkhorn distance efficiently is a well studied problem in a number of communities.
The Sinkhorn distance is so named because, as was pointed out by~\citet{Cut13}, there is an extremely simple iterative algorithm due to~\citet{Sin67} which converges quickly to a solution to~\eqref{eq:peta}. This algorithm, which we call \emph{Sinkhorn scaling}, works very well in practice and can be implemented using only matrix-vector products, which makes it easily parallelizable. Sinkhorn scaling has been analyzed many times~\citep{FraLor89,LSW98,KalLarRic08,AltWeeRig17,DvuGasKro18}, and forms the basis for the first algorithms for the unregularized optimal transport problem that run in time nearly linear in the size of the cost matrix~\citep{AltWeeRig17,DvuGasKro18}. Greedy and stochastic algorithms related to Sinkhorn scaling with better empirical performance have also been explored~\citep{GenCutPey16,AltWeeRig17}.
Another influential technique, due to~\citet{SolGoePey15}, exploits the fact that, when the distributions are supported on a grid, Sinkhorn scaling performs extremely quickly by decomposing the cost matrix along lower-dimensional slices.

Other algorithms have sought to solve~\eqref{eq:peta} by bypassing Sinkhorn scaling entirely.
\citet{BlaJamKen18} proposed to solve~\eqref{eq:peta} directly using second-order methods based on fast Laplacian solvers~\citep{CohMpoTsi17,AllLiOli17}. \citet{BlaJamKen18} and \citet{quanrud} have noted a connection to packing linear programs, which can also be exploited to yield near-linear time algorithms for unregularized transport distances.

Our main algorithm relies on constructing a low-rank approximation of a Gaussian kernel matrix from a small subset of its columns and rows. Computing such approximations is a problem with an extensive literature in machine learning, where it has been studied under many different names, e.g., Nystr\"om method~\citep{williams2001using}, sparse greedy approximations~\citep{smola2000sparse}, incomplete Cholesky decomposition~\citep{fine01efficient}, Gram-Schmidt orthonormalization~\citep{Cristianini2004} or CUR matrix decompositions \citep{mahoney2009cur}. The  approximation properties of these algorithms are now well understood~\citep{mahoney2009cur,gittens2011spectral,bach2013sharp,alaoui2015fast}; however, in this work, we require significantly more accurate bounds than are available from existing results as well as adaptive bounds for low-dimensional data. To establish these guarantees, we follow an approach based on approximation theory~\citep[see, e.g.,][]{rieger2010sampling,wendland2004scattered,belkin}, which consists of analyzing interpolation operators for the reproducing kernel Hilbert space corresponding to the Gaussian kernel.

Finally, this paper adds to recent work proposing the use of low-rank approximation for Sinkhorn scaling~\citep{AltBacRud18,TenWolKim18}.
We improve upon those papers in several ways. First, although we also exploit the idea of a low-rank approximation to the kernel matrix, we do so in a more sophisticated way that allows for automatic adaptivity to data with low-dimensional structure. These new approximation results are the key to our adaptive algorithm, and this yields a significant improvement in practice.
Second, the analyses of~\citet{AltBacRud18} and \citet{TenWolKim18} only yield an approximation to $\Obj$ when $\eta \to \infty$. In the moderately regularized case when $\eta = O(1)$, which is typically used in practice, neither the work of \citet{AltBacRud18} nor of \citet{TenWolKim18} yields a rigorous error guarantee.

\subsection{Outline of paper}\label{subsec:intro-outline}
Section~\ref{sec:statement} recalls preliminaries, and then formally states our main result and gives pseudocode for our proposed algorithm. The core of our theoretical analysis is in Sections~\ref{sec:nystrom} and~\ref{sec:sink}. Section~\ref{sec:nystrom} presents our new results for Nystr\"om approximation of Gaussian kernel matrices and Section~\ref{sec:sink} presents our new stability results for Sinkhorn scaling. Section~\ref{sec:alg} then puts these results together to conclude a proof for our main result (Theorem~\ref{thm:main}). Finally, Section~\ref{sec:experimental} contains experimental results showing that our proposed algorithm outperforms state-of-the-art methods. The appendix contains proofs of several lemmas that are deferred for brevity of the main text.

\section{Main result}\label{sec:statement}
\subsection{Preliminaries and notation}\label{sec:prelim}

\paragraph{Problem setup.} Throughout, $\p$ and $\q$ are two probability distributions supported on a set $X \defeq \{x_1, \dots, x_n\}$ of points in $\RR^d$, with $\|x_i\|_2 \leq R$ for all $i \in [n] := \{1, \dots, n\}$. We define the cost matrix $C \in \RR^{n \times n}$ by $C_{ij} = \|x_i - x_j\|_2^2$.
We identify $\p$ and $\q$ with vectors in the simplex $\Delta_n := \{v \in \R_{\geq 0}^n : \sum_{i=1}^n v_i = 1\}$ whose entries denote the weight each distribution gives to the points of $X$.
We denote by $\Coup$ the set of couplings between $\p$ and $\q$, identified with the set of $P \in \RR_{\geq 0}^{n \times n}$ satisfying $P \bone = \p$ and $P^\top \bone = \q$, where $\bone$ denotes the all-ones vector in $\RR^n$. The Shannon entropy of a non-negative matrix $P \in \RR_{\geq 0}^{n \times n}$ is denoted $H(P) \defeq \sum_{ij} P_{ij} \log \tfrac{1}{P_{ij}}$, where we adopt the standard convention that $0 \log\tfrac{1}{0} = 0$.

Our goal is to approximate the Sinkhorn distance with parameter $\eta > 0$:
\begin{equation*}
\Obj
\defeq \min_{P \in \Coup} \sum_{ij} P_{ij} \|x_i - x_j\|_2^2 - \eta^{-1} \ent(P)
\end{equation*}
to some additive accuracy $\eps > 0$.
By strict convexity, this optimization problem has a unique minimizer, which we denote henceforth by $\Peta$. For shorthand, in the sequel we write
\[
V_{M}(P) := \langle M, P \rangle - \eta^{-1} \ent(P),
\]
for a matrix $M \in \RR^{n \times n}$. In particular, we have $\Obj = \min_{P \in \Coup} V_C(P)$.
For the purpose of simplifying some bounds, we assume throughout that $n \geq 2$, $\eta \in [1,n]$, $R \geq 1$, $\eps \leq 1$.

\paragraph{Sinkhorn scaling.} Our approach is based on Sinkhorn scaling, an algorithm due to~\citet{Sin67} and popularized for optimal transport by~\citet{Cut13}.
We recall the following fundamental definition.
\begin{defin}
Given $\p, \q \in \Delta_n$ and $K \in \RR^{n \times n}$ with positive entries, the \emph{Sinkhorn projection}~$\Sinkcoup(K)$ of $K$ onto $\Coup$ is the \emph{unique} matrix in $\Coup$ of the form $D_1 K D_2$ for positive diagonal matrices $D_1$ and $D_2$.
\end{defin}
Since $\p$ and $\q$ remain fixed throughout, we abbreviate $\Sinkcoup$ by $\Sink$ except when we want to make the feasible set $\Coup$ explicit.

\begin{prop}[\citealp{Wil69}]\label{prop:kkt}
Let $K$ have strictly positive entries, and let $\log K$ be the matrix defined by $(\log K)_{ij} \defeq \log (K_{ij})$.
Then
\begin{equation*}
\Sinkcoup(K) = \argmin_{P \in \Coup} \langle - \log K, P \rangle - H(P)\,.
\end{equation*}
Note that the strict convexity of $- H(P)$ and the compactness of $\Coup$ implies that the minimizer exists and is unique.
\end{prop}

This yields the following simple but key connection between Sinkhorn distances and Sinkhorn scaling.

\begin{cor}\label{cor:sink-basic}
\begin{equation*}
\Peta = \Sinkcoup(K)\,,
\end{equation*}
where $K$ is defined by $K_{ij} = e^{- \eta C_{ij}}$.
\end{cor}

\citet{Sin67} proposed to find $\Sink(K)$ by alternately renormalizing the rows and columns of $K$. This well known algorithm has excellent performance in practice, is simple to implement, and is easily parallelizable since it can be written entirely in terms of matrix-vector products~\citep[Section~4.2]{PeyCut17}. Pseudocode for the version of the algorithm we use can be found in Appendix~\ref{app:sink-code}.

\paragraph{Notation.}
We define the probability simplices $\Delta_n \defeq \{p \in \RR_{\geq 0}^n: p^\top \bone = 1\}$ and $\Delta_{n \times n} \defeq \{P \in \RR_{\geq 0}^{n \times n}: \bone^\top P \bone = 1\}$. Elements of $\Delta_{n \times n}$ will be called \emph{joint distributions}. The Kullback-Leibler divergence between two joint distributions $P$ and $Q$ is $\KL{P}{Q} \defeq \sum_{ij} P_{ij} \log \tfrac{P_{ij}}{Q_{ij}}$.

Throughout the paper, all matrix exponentials and logarithms will be taken entrywise, i.e., $(e^{A})_{ij} \defeq e^{A_{ij}}$ and $(\log A)_{ij} \defeq \log A_{ij}$ for $A \in \RR^{n \times n}$.

Given a matrix $A$, we denote by $\|A\|_{\mathrm{op}}$ its operator norm (i.e., largest singular value), by $\|A\|_*$ its nuclear norm (i.e., the sum of its singular values), by $\|A\|_1$ its entrywise $\ell_1$ norm (i.e., $\|A\|_1 \defeq \sum_{ij} |A_{ij}|$), and by $\|A\|_{\infty}$ its entrywise $\ell_\infty$ norm (i.e., $\|A\|_{\infty} \defeq \max_{ij} |A_{ij}|$). We abbreviate ``positive semidefinite'' by ``PSD.''

The notation $f = O(g)$ means that $f \leq C g$ for some universal constant $C$, and $g = \Omega(f)$ means $f = O(g)$. The notation $\Otilde(\cdot)$ omits polylogarithmic factors depending on $R$, $\eta$, $n$, and $\eps$.

%
%
%

\subsection{Main result and proposed algorithm}
Pseudocode for our proposed algorithm is given in Algorithm~\ref{alg:main}. 
\textsc{Nys-Sink} (pronounced ``nice sink'') computes a low-rank Nystr\"om approximation of the kernel matrix via a column sampling procedure.
While explicit low-rank approximations of Gaussian kernel matrices can also be obtained via Taylor explansion~\citep{cotter2011explicit}, our approach automatically adapts to the properties of the data set, leading to much better performance in practice.
%


As noted in Section~\ref{sec:intro}, the Nystr\"om method constructs a low-rank approximation to a Gaussian kernel matrix $K = e^{-\eta C}$ based on a small number of its columns. In order to design an efficient algorithm, we aim to construct such an approximation with the smallest possible rank. The key quantity for understanding the error of this algorithm is the so-called \emph{effective dimension} (also sometimes called the ``degrees of freedom'') of the kernel $K$~\citep{hastie,zhang2005learning,musco2017recursive}.

\begin{defin}
Let $\la_j(K)$ denote the $j$th largest eigenvalue of $K$ (with multiplicity). Then the \emph{effective dimension} of $K$ at level $\tau > 0$ is
\eqal{
    \deff(\tau) := \sum_{j=1}^n \frac{\la_j(K)}{\la_j(K) + \tau n}.
}
\end{defin}
The effective dimension $\deff(\tau)$ indicates how large the rank of an approximation $\tilde K$ to $K$ must be in order to obtain the guarantee $\|\tilde K - K\|_{\mathrm{op}} \leq \tau n$.
As we will show in Section~\ref{sec:alg}, below, it will suffice for our application to obtain an approximate kernel $\tilde K$ satisfying $\|\tilde K - K\|_{\mathrm{op}} \leq \tfrac{\ep'}{2} e^{-4 \eta R^2}$, where $\ep' = \tilde O(\ep R^{-2})$.
We are therefore motivated to define the following quantity, which informally captures the smallest possible rank of an approximation of this quality.
\begin{defin}\label{def:deff}
Given $X = \{x_1, \dots, x_n\} \subseteq \RR^d$ with $\|x_i\|_2 \leq R$ for all $i \in [n]$, $\eta > 0$, and $\eps' \in (0, 1)$, the \emph{approximation rank} is
\begin{equation*}
r^*(X, \eta, \eps') \defeq 
\deff\left(\tfrac{\ep'}{2n} e^{-4 \eta R^2}\right)
\end{equation*}
where $\deff(\cdot)$ is the effective rank for the kernel matrix~$K \defeq e^{-\eta C}$.
\end{defin}
\par As we show below, we adaptively construct an approximate kernel $\tilde K$ whose rank is at most a logarithmic factor bigger than $r^*(X, \eta, \eps')$ with high probability. We also give concrete bounds on $r^*(X, \eta, \eps')$ below. 

Our proposed algorithm makes use of several subroutines. The $\AdaptiveNsytrom$ procedure in line~\ref{line:nystrom} combines an algorithm of~\citet{musco2017recursive} with a doubling trick that enables automatic adaptivity; this is described in Section~\ref{sec:nystrom}. It outputs the approximate kernel $\tilde{K}$ and its rank $r$. The $\Sinkhorn$ procedure in line~\ref{line:sinkhorn} is the Sinkhorn scaling algorithm for projecting $\tilde{K}$ onto $\Coup$, pseudocode for which can be found in Appendix~\ref{app:sink-code}. We use a variant of the standard algorithm, which returns both the scaling matrices and an approximation of the cost of an optimal solution. The $\Round$ procedure in line~\ref{line:round} is Algorithm~2 of~\citet{AltWeeRig17}; for completeness, pseudocode can be found here in Appendix~\ref{app:round-code}.


We emphasize that neither $D_1\tilde{K}D_2$ nor $\Phat$ (which is of the form $D_1' \tilde{K} D_2' + vw^T$ for diagonal matrices $D_1',D_2'$ and vectors $v,w$) are ever represented explicitly, since this would take $\Omega(n^2)$ time. Instead, we maintain these matrices in low-rank factorized forms. This enables Algorithm~\ref{alg:main} to be implemented efficiently in $o(n^2)$ time, since the procedures $\Sinkhorn$ and $\Round$ can both be implemented such that they depend on $\tilde{K}$ only through matrix-vector multiplications with $\tilde{K}$. Moreover, we also emphasize that all steps of Algorithm~\ref{alg:main} are easily parallelizable since they can be re-written in terms of matrix-vector multiplications.

We note also that although the present paper focuses specifically on the squared Euclidean cost $c(x_i, x_j) = \|x_i - x_j\|_2^2$ (corresponding to the $2$-Wasserstein case of optimal transport pervasively used in applications; see intro), our algorithm \textsc{Nys-Sink} readily extends to other cases of optimal transport. Indeed, since the Nystr\"om method works not only for Gaussian kernel matrices $K_{ij} = e^{-\eta \|x_i - x_j\|_2^2}$, but in fact more generally for any PSD kernel matrix, our algorithm can be used on any optimal transport instance for which the corresponding kernel matrix $K_{ij} = e^{-\eta c(x_i,x_j)}$ is PSD.

\begin{algorithm}[H]
\begin{algorithmic}[1]
\Require{$X = \{x_1, \dots, x_n\} \subseteq \RR^{d}$, $\p,\q \in \Delta_n$, $\eps,\eta > 0$}
\Ensure{$\hat P \in \Coup$, $\hat W \in \RR$, $r \in \NN$}
\State $\eps' \gets \min(1, \tfrac{\eps \eta}{50(4 R^2\eta + \log \tfrac{n}{\eta \eps})})$
\State $(\tilde K, r) \gets \AdaptiveNsytrom(X,\eta,\tfrac{\eps'}{2} e^{-4\eta R^2})$ \label{line:nystrom} \Comment Compute low-rank approximation
\State $(D_1, D_2, \hat W) \gets \Sinkhorn(\tilde K, \p, \q, \eps')$ \label{line:sinkhorn} \Comment Approximate Sinkhorn projection and cost
\State $\Phat \gets \Round(D_1 \tilde K D_2, \p, \q)$
\Comment Round to feasible set
\label{line:round}
\State \Return{$\Phat$,  $\hat W$}
\end{algorithmic}
\caption{\textsc{Nys-Sink}}
\label{alg:main}
\end{algorithm}

\par Our main result is the following.

\begin{theorem}\label{thm:main}
Let $\eps,\delta \in (0,1)$. 
Algorithm~\ref{alg:main} 
runs in
$\tilde{O}\left( nr \left( r + \tfrac{\eta R^4}{\eps} \right) \right)$
time, uses $O(n(r+d))$ space, and returns a feasible matrix $\Phat \in \Coup$ in factored form and scalars $\hat W \in \R$ and $r \in \N$,  where
\begin{subequations}
\begin{align}
|V_C(\Phat) - \Obj| & \leq \eps, \label{eq:thm-main:close-obj}
\\ \KL{\Phat}{\Peta} & \leq \eta \eps, \label{eq:thm-main:close-kl} \\
|\hat W - \Obj| & \leq \eps,\, \label{eq:thm-main:approx-val}
\end{align}
and, with probability $1 - \delta$,
\begin{align}
r & \leq c \cdot r^*(X, \eta, \eps') \log \tfrac{n}{\delta} \,,
\label{eq:thm-main:r}
\end{align}
for a universal constant $c$ and where $\eps' = \tilde \Omega(\eps R^{-2})$.
\end{subequations}

\end{theorem}
We note that, while our algorithm is randomized, we obtain a deterministic guarantee that $\hat P$ is a good solution. 
We also note that runtime dependence on the radius $R$---which governs the scale of the problem---is inevitable since we seek an additive guarantee.

Crucially, we show in Section~\ref{sec:nystrom} that $r^*$---which controls the running time of the algorithm with high probability by~\eqref{eq:thm-main:r}---adapts to the \emph{intrinsic} dimension of the data. This adaptivity is crucial in applications, where data can have much lower dimension than the ambient space. We informally summarize this behavior in the following theorem.

\begin{theorem}[Informal]\label{thm:rstar}
\begin{enumerate}
\item There exists an universal constant $c > 0$ such that, for any $n$ points in a ball of radius $R$ in $\RR^d$,
\begin{equation*}
r^*(X, \eta, \eps') \leq (c(\eta R^2 + \log \tfrac{n}{\eps' \eta}))^d\,.
\end{equation*}
\item For any $k$-dimensional manifold $\Omega$ satisfying certain technical conditions and $\eta > 0$, there exists a constant $c_{\Omega, \eta}$ such that for any $n$ points lying on $\Omega$,
\begin{equation*}
r^*(X, \eta, \eps') \leq c_{\Omega, \eta} (\log \tfrac{n}{\eps'})^{5k/2}\,.
\end{equation*}
\end{enumerate}
\end{theorem}
The formal versions of these bounds appear in Section~\ref{sec:nystrom}.
The second bound is significantly better than the first when $k \ll d$, and clearly shows the benefits of an adaptive procedure.

Combining Theorems~\ref{thm:main} and~\ref{thm:rstar} yields the following time and space complexity for our algorithm.

\begin{cor}[Informal]
If $X$ consists of $n$ points lying in a ball of radius $R$ in $\R^d$, then with high probability Algorithm~\ref{alg:main} requires 
\begin{equation*}
   \tilde O\left(n \cdot \frac{1}{\eps}\left(c\eta R^2 + c\log \frac{n}{\eps}\right)^{2d+1}\right) \text{ time and } \tilde O\left(n \cdot \left(c\eta R^2+ c\log \frac{n}{\eps}\right)^{d}\right) \text{ space.}
\end{equation*}

Moreover, if $X$ lies on a $k$-dimensional manifold $\Omega$, then with high probability 
Algorithm~\ref{alg:main} requires 
\begin{equation*}
   \tilde O\left(n \cdot \frac{c_{\Omega, \eta}}{\eps}\left(\log \frac n \eps\right)^{5k}\right) \text{ time and } \tilde O\left(n \cdot c_{\Omega, \eta} \left(\log \frac n \eps\right)^{5k/2}\right) \text{ space.}
\end{equation*}
\end{cor}

\citet{AltWeeRig17} noted that an approximation to the \emph{unregularized} optimal transport cost is obtained by taking $\eta = \Theta\left(\eps^{-1}\log n \right)$. Thus it follows that
Algorithm~\ref{alg:main} computes an additive $\eps$ approximation to the unregularized transport distance in $O\left(n\left(\eps^{-1} R^2 \log n\right)^{O(d)}\right)$ time with high probability. However, a theoretically better running time for that problem can be obtained by a simple but impractical algorithm based on rounding the input distributions to an $\eps$-net and then running Sinkhorn scaling on the resulting instance.\footnote{We are indebted to Piotr Indyk for inspiring this remark.}

\section{Kernel approximation via the Nystr\"om method}\label{sec:nystrom}

In this section, we describe the algorithm $\AdaptiveNsytrom$ used in line~\ref{line:nystrom} of Algorithm~\ref{alg:main} and bound its runtime complexity, space complexity, and error. We first establish basic properties of Nystr\"om approximation and give pseudocode for \AdaptiveNsytrom~(Sections~\ref{subsec:nystrom:overview} and~\ref{subsec:nystrom:doubling}) before stating and proving formal versions of the bounds appearing in Theorem~\ref{thm:rstar}~(Sections~\ref{subsec:nystrom:wc} and~\ref{subsec:nystrom:manifold}).



\subsection{Preliminaries: Nystr\"om and error in terms of effective dimension}\label{subsec:nystrom:overview}
Given points $X = \{x_1, \dots, x_n\}$ with $\|x_i\|_2 \leq R$ for all $i \in [n]$, let $K \in \R^{n\times n}$ denote the matrix with entries $K_{ij} := k_\eta(x_i,x_j)$, where $k_\eta(x,x') := e^{-\eta \|x - x'\|^2}$. Note that $k_{\eta}(x,x')$ is the Gaussian kernel $e^{-\|x-x'\|^2/(2\sigma^2)}$ between points $x$ and $x'$ with bandwith parameter $\sigma^2 = \tfrac{1}{2\eta}$.
For $r \in \N$, we consider an approximation of the matrix $K$ that is of the form
$$\Kt = V A^{-1} V^\top,$$
where $V \in \R^{n \times r}$ and $A \in \R^{r \times r}$.
In particular we will consider the approximation given by the Nystr\"om method which, given a set $X_r = \{\widetilde{x}_1,\dots,\widetilde{x}_r\} \subset X$, constructs $V$ and $A$ as:
$$V_{ij} = k_\eta(x_i,\widetilde{x}_j), \quad A_{jj'} = k_\eta(\widetilde{x}_j,\widetilde{x}'_j),$$
for $i \in [n]$ and $j,j' \in [r]$. 
Note that the matrix $\Kt$ is never computed explicitly. Indeed, our proposed Algorithm~\ref{alg:main} only depends on $\Kt$ through computing matrix-vector products $\Kt v$, where $v \in \R^n$, and these can be computed efficiently as
\begin{align}
\Kt v = V(L^{-\top}(L^{-1}(V^\top v))),
\label{eq:Ktilde-decomp}
\end{align}
where $L \in \R^{r \times r}$ is the lower triangular matrix satisfying $LL^\top = A$ obtained by the Cholesky decomposition of $A$, and where we compute products of the form $L^{-1} v$ (resp. $L^{-\top}v$) by solving the triangular system $Lx = v$ (resp. $L^{\top}x = v$).
Once a Cholesky decomposition of $A$ has been obtained---at computational cost $O(r^3)$---matrix-vector products can therefore be computed in time $O(nr)$.
%

We now turn to understanding the approximation error of this method. In this paper we will sample the set $X_r$ via 
{\em approximate leverage-score sampling}.
In particular, we do this via Algorithm 2 of \citet{musco2017recursive}.
The following lemma shows that taking the rank $r$ to be on the order of the effective dimension $\deff(\tau)$ (see Definition~\ref{def:deff}) is sufficient to guarantee that $\widetilde{K}$ approximates $K$ to within error $\tau$ in operator norm.

\begin{lemma}\label{lm:exists-algo-nystrom}
	Let $\tau,\delta > 0$. Consider sampling $X_r$ from $X$ according to Algorithm 2 of \citet{musco2017recursive}, for some positive integer $r \geq 400 \deff(\tau) \log \tfrac{3 n}{\delta}.$
	Then:
	\begin{enumerate}
	    \item Sampling $X_r$ and forming the matrices $V$ and $L$ (which define $\tilde{K}$, see~\eqref{eq:Ktilde-decomp}) requires $O(nr^2 + r^3)$ time and $O(n(r+d))$ space.
	    \item Computing matrix-vector products with $\tilde{K}$ can be done in time $O(nr)$.
	    \item With probability at least $1 - \delta$, $\|K - \widetilde{K}\|_{\mathrm{op}} \leq \tau n$.
	\end{enumerate}
\end{lemma}
\begin{proof}
The result follows directly from Theorem~7 of \citet{musco2017recursive} and 
    the fact that $\deff(\tau) \leq \textrm{rank}(K) \leq n$ for any $\tau \geq 0$.
\end{proof}

\subsection{Adaptive Nystr\"om with doubling trick}\label{subsec:nystrom:doubling}

Here we give pseudocode for the \textsc{AdaptiveNystr\"om} subroutine in Algorithm~\ref{alg:adaptive-nystrom}. The algorithm is based on a simple doubling trick, so that the rank of the approximate kernel can be chosen adaptively.
The observation enabling this trick is that given a Nystr\"om approximation $\tilde{K}$ to the actual kernel matrix $K = e^{-\eta\|x_i-x_j\|_2^2}$, the entrywise error $\|K - \tilde{K}\|_{\infty}$ of the approximation can be computed exactly in $O(nr^2)$ time. The reason for this is that (i) the entrywise norm $\|K - \tilde{K}\|_{\max}$ is equal to the maximum entrywise error on the diagonal $\max_{i \in [n]} |K_{ii} - \tilde{K}_{ii}| = 1 - \min_{i \in [n]} \tilde{K}_{ii}$, proven below in Lemma~\ref{lem:nystrom-doubling}; and (ii) the quantity $1 - \min_{i \in [n]} \tilde{K}_{ii}$ is easy to compute quickly.

Below, line~\ref{nystrom:line} in Algorithm~\ref{alg:adaptive-nystrom} denotes the approximate leverage-score sampling scheme of~\citet[Algorithm~2]{musco2017recursive} when applied to the Gaussian kernel matrix $K_{ij} := e^{-\eta \|x_i - x_j\|^2}$. We note that the BLESS algorithm of~\citet{rudi2018fast} allows for re-using previously sampled points when doubling the sampling rank. Although this does not affect the asymptotic runtime, it may lead to speedups in practice.

\begin{algorithm}[h]
\begin{algorithmic}[1]
\Require{$X = \{x_1, \dots, x_n\} \in \RR^{d \times n}$, $\eta > 0$, $\tau > 0$}
\Ensure{$\tilde{K} \in \R^{n \times n}$, $r \in \mathbb{N}$}
\State $\mathsf{err} \gets +\infty$, $r \gets 1$
\While{$\mathsf{err} > \tau$}
\State $r \gets 2 r$
\State $\tilde K \gets \Nystrom(X,\eta, r)$
\label{nystrom:line}
\State $\mathsf{err} \gets 1 - \min_{i \in [n]} \tilde{K}_{ii}$
\EndWhile
\State \Return{$(\tilde{K}, \rank(\tilde{K}))$}
\end{algorithmic}
\caption{\AdaptiveNsytrom}
\label{alg:adaptive-nystrom}
\end{algorithm}

\begin{lemma}\label{lem:nystrom-doubling}
Let $(\tilde{K}, r)$ denote the (random) output of $\AdaptiveNsytrom(X,\eta,\tau)$. Then:
\begin{enumerate}
\item $\|K -\tilde{K}\|_{\infty} \leq \tau$.
\item The algorithm used $O(nr)$ space and terminated in $O(nr^2)$ time.
\item There exists a universal constant $c$ such that simultaneously for every $\delta > 0$,
\[
\mathbb{P}\Big(r \leq c \cdot
\deff\left(\tfrac{\tau}{n}\right)
\log \left(\tfrac{n}{\delta}\right) \Big) \geq 1 - \delta.
\]
\end{enumerate}
\end{lemma}
\begin{proof}
By construction, the Nystr\"om approximation $\tilde{K}$ is a PSD approximation of $K$ in the sense that $K \succeq \tilde{K} \succeq 0$, see e.g.,~\citet[Theorem 3]{musco2017recursive}. Since Sylvester's criterion for $2 \times 2$ minors guarantees that the maximum modulus entry of a PSD matrix is always achieved on the diagonal, it follows that $\|K - \tilde{K}\|_{\infty} = \max_{i \in [n]} |K_{ii} - \tilde{K}_{ii}|$. Now each $K_{ii} = 1$ by definition of $K$, and each $\tilde{K}_{ii} \in [0,1]$ since $K \succeq \tilde{K} \succeq 0$. Therefore we conclude
\[
\|K - \tilde{K}\|_{\infty} = 1 - \min_{i \in [n]} \tilde{K}_{ii}.
\]
This implies Item 1. Item 2 follows upon using the space and runtime complexity bounds in Lemma~\ref{lm:exists-algo-nystrom} and noting that the final call to $\Nystrom$ is the dominant for both space and runtime. Item 3 is immediate from Lemma~\ref{lm:exists-algo-nystrom} and the fact that $\|K - \tilde{K}\|_{\infty} \leq \|K - \tilde{K}\|_{\mathrm{op}}$ (Lemma~\ref{lem:norm-equiv}).
\end{proof}

\subsection{General results: data points lie in a ball}\label{subsec:nystrom:wc}
In this section we assume no structure on $X$ apart from the fact that $X \subseteq B^d_R$ where $B^d_R$ is a ball of radius $R$ in $\R^d$ centered around the origin, for some $R > 0$ and $d \in \N$. First we characterize the eigenvalues of $K$ in terms of $\eta, d, R$, and then we use this to bound $\deff$.

\begin{theorem}\label{thm:eig-ball}
		Let $X := \{x_1, \dots x_n\} \subseteq B_R^d$, and let $K \in \R^{n \times n}$ be the matrix with entries $K_{ij} := e^{-\eta\|x_i - x_j\|^2}$. Then:
		\begin{enumerate}
		    \item For each $t \in \N, t \geq (2e)^d$, $\la_{t+1}(K) \leq n e^{-\frac{d}{2e} t^{1/d} \log {d~ t^{1/d} \over 4 e^2 \eta R^2}}$.
		    \item For each $\tau \in (0, 1]$, $\deff(\tau) \leq 3 \left(6+\frac{41}{d}\eta R^2+ \frac{3}{d}\log \frac{1}{\tau}\right)^d$.
		\end{enumerate}
\end{theorem}

We sketch the proof of Theorem~\ref{thm:eig-ball} here; details are deferred to Appendix~\ref{sec:nystrom-app} for brevity of the main text. We begin by recalling the argument of~\citet{cotter2011explicit} that truncating the Taylor expansion of the Gaussian kernel guarantees for each positive integer $T$ the existence of a rank $M_T := \binom{d+T}{T}$ matrix $\tilde{K}_T$ satisfying
\[
\|K - \tilde{K}_T \|_{\infty} \leq \frac{(2\eta R^2)^{T+1}}{(T+1)!}.
\]
On the other hand, by the Eckart-Young-Mirsky Theorem,
\[
\lambda_{M_T+1}
=
\inf_{\bar{K}_T \in \RR^{n \times n},\, \rank(\bar{K}_T) \leq M_T} \|K - \bar{K}_T \|_{\op}.
\]
Therefore by combining the above two displays, we conclude that
\[
\lambda_{M_T+1}
\leq
\|K - \tilde{K}_T \|_{\op}
\leq n \|K - \tilde{K}_T\|_{\infty} \leq n \frac{(2\eta R^2)^{T+1}}{(T+1)!}.
\]
Proofs of the two claims follow by bounding this quantity. Details are in Appendix~\ref{sec:nystrom-app}.

\cref{thm:eig-ball} characterizes the eigenvalue decay and effective dimension of Gaussian kernel matrices in terms of the dimensionality of the space, with explicit constants and explicit dependence on the width parameter $\eta$ and the radius $R$ of the ball \citep[see][for asymptotic results]{belkin}. This yields the following bound on the optimal rank for approximating Gaussian kernel matrices of data lying in a Euclidean ball.

\begin{cor}\label{thm:nystrom-ball}
Let $\eps' \in (0, 1)$ and $\eta > 0$.
If $X$ consists of $n$ points lying in a ball of radius $R$ around the origin in $\RR^d$, then
\begin{equation*}
r^*(X, \eta, \eps') \leq 3\left(6 + \frac{53}{d} \eta R^2 + 
\frac 3d \log \frac {2 n}{\eps'}\right)^d
\end{equation*}
\end{cor}
\begin{proof}
Directly from the explicit bound of Theorem~\ref{thm:eig-ball} and the definition of $r^*(X, \eta, \eps')$.
\end{proof}


\subsection{Adaptivity: data points lie on a low dimensional manifold}\label{subsec:nystrom:manifold}
In this section we consider $X \subset \Omega \subset \R^d$, where $\Omega$ is a low dimensional manifold.
In \cref{thm:eig-manifold} we give a result about the approximation properties of the  Gaussian kernel over manifolds and a bound on the eigenvalue decay and effective dimension of Gaussian kernel matrices. We prove that the effective dimension is logarithmic in the precision parameter $\tau$ to a power depending only on the dimensionality $k$ of the manifold (to be contrasted to the dimensionality of the ambient space $d \gg k$).

Let $\Omega \subset \R^d$ be a smooth compact manifold without boundary, and $k < d$.
Let $(\Psi_j, U_j)_{j \in [T]}$, with $T \in \N$, be an atlas for $\Omega$, where without loss of generality, $(U_j)_j$ are open sets covering $\Omega$, $\Psi_j:U_j \to B^k_{r_j}$ are smooth maps with smooth inverses, mapping $U_j$ bijectively to $B^k_{r_j}$, balls of radius $r_j$ centered around the origin of $\R^k$. We assume the following quantitative control on the smoothness of the atlas.
\begin{assumption}\label{asm:good-manifold}
There exists $Q > 0$ such that 
$$\sup_{u \in B^k_{r_j}} \|D^\alpha \Psi_j^{-1}(u)\| \leq Q^{|\alpha|}, \qquad \alpha \in \N^k, j \in [T],$$
where $|\alpha| = \sum_{j=1}^k \alpha_j$ and $D^\alpha = \frac{\partial^{|\alpha|}}{\partial u_1^{\alpha_1} \dots \partial u_k^{\alpha_k}}$, for $\alpha \in \N^k$.
\end{assumption}
Before stating our result, we need to introduce the following helpful definition. Given $f: \RR^d \to \RR$, and $X = \{x_1, \dots, x_n\} \subset \Omega$, denote by $\widehat{f}_X$ the function
$$\widehat{f}_X(x) := \sum_{i = 1}^n c_i k_\eta(x,x_i), \quad c = K^{-1}v_f,$$
with $v_f = (f(x_1),\dots, f(x_n))$ and $K \in \R^{n \times n}$ the kernel matrix over $X$, i.e. $K_{ij} = k_\eta(x_i, x_j)$. Note that $\widehat{f}_X(x_i) = f(x_i)$ by construction \citep{wendland2004scattered}. We have the following result.

\begin{theorem}\label{thm:eig-manifold}
Let $\Omega \subset B_R^d \subset \R^d$ be a smooth compact manifold without boundary satisfying \cref{asm:good-manifold}.
Let $X \subset \Omega$ be a set of cardinality $n \in \N$. Then the following holds
\begin{enumerate}
\item  Let $h_{X,\Omega} = \sup_{x' \in \Omega} \inf_{x \in X} \|x-x'\|$. Let $\hh$ be the RKHS associated to the Gaussian kernel of a given width. There exist $c, h > 0$ not depending on $X, n$, such that, when $h_{X, \Omega} \leq h$ the following holds
$$|f(x) - \widehat{f}_X(x)| \leq e^{-c h_{X,\Omega}^{-2/5}}\|f\|_\hh, \qquad \forall f \in \hh, ~~x \in \Omega.$$
\item Let $K$ be the Gaussian kernel matrix associated to $X$. Then there exists a constant $c$ not depending on $X$ or $n$, for which
$$\la_{p+1}(K) \leq n e^{-c p^{\frac{2}{5k}}}, \qquad \forall p \in [n].$$
\item Let $\tau \in (0,1]$. Let $K$ be the Gaussian kernel matrix associated to $X$ and $\deff(\tau)$ the effective dimension computed on $K$. There exists $c_1,c_2$ not depending on $X$, $n$, or $\tau$, for which
$$\deff(\tau) \leq \left(c_1 \log \frac{1}{\tau}\right)^{5k/2} + c_2.$$
\end{enumerate}
\end{theorem}
\begin{proof}
First we recall some basic multi-index notation and introduce Sobolev Spaces. When $\alpha \in \N_0^d, x \in \R^d, g:\R^d \to \R$, we write 
$$x^\alpha = \prod_i x_i^{\alpha_i},~~ |\alpha| = \sum_i \alpha_i,~~ \alpha! = \prod_i \alpha_i!,~~ D^\alpha = \frac{\partial^{|\alpha|}}{\partial x_1^{\alpha_1}\dots \partial x_n^{\alpha_n}}.$$ 
Next, we recall the definition of Sobolev spaces. For $m,p \in \N$ and $B \subseteq \R^k$, define the norm $\|\cdot\|_{W^m_p(B)}$ by
$$\|f\|^p_{W^m_p(B)} = \sum_{|\alpha| \leq m} \|D^\alpha f\|^p_{L^p(B)},$$
and the space of $W^m_p(B)$ as $W^m_p(B) = \overline{C^\infty(B)}^{\|\cdot\|_{W^m_p(B)}}$.

For any $j \in [T]$, $u \in \hh$ we have the following. By \cref{lm:bound-composite-sobolev}, we have that there exists a constant $C_{d,k,R,r_j}$ such that for any $q \geq k$,
$$\|u \circ \Psi_j^{-1}\|_{W^q_2(B^k_{r_j})} \leq C_{d,k,R,r_j} q^k (2q d Q)^{q}\| u\|_{W^{q + (d+1)/2}_2(B^d_R)}.$$
Now note that by Theorem~7.5 of \citet{rieger2010sampling} we have that there exists a constant $C_\eta$ such that
$$\| u\|_{W^{q + (d+1)/2}_2(B^d_R)} \leq \| u\|_{W^{q + (d+1)/2}_2(\R^d)} \leq (C_\eta)^{q + (d+1)/2} \left(q + \frac{d+1}{2}\right)^{\frac{q}{2} + \frac{d+1}{4}} \|u\|_\hh.$$
Then, since $q^m \leq m^m (1+m)^q$, for any $q \geq 1$, we have 
$$\|u \circ \Psi_j^{-1}\|_{W^q_2(B^k_{r_j})} \leq  C_{d,k,R,r_j,Q,\eta}^q  q^{\frac{3q}{2}} \|u\|_\hh,$$
for a suitable constant $C_{d,k,R,r_j,Q,\eta}$ depending on $C_{d,k,R,r_j,Q}$, $C_\eta$ and $(d+1)/2$.

In particular we want to study $\|u\|_{L^\infty(\Omega)}$, for $u = f - \widehat{f}_X$. We have
$$\|u\|_{L^\infty(\Omega)} = \sup_{j \in [T]} \|u\|_{L^\infty(U_j)} = \sup_{j \in [T]} \|u \circ \Psi_j^{-1}\|_{L^\infty(B^k_{r_j})}.$$

Now for $j \in [T]$, denote by $Z_j$ the set $Z_j = \{\Psi_j(x) | x \in X \cap U_j \}$. By construction of $u = f - \widehat{f}_X$, we have
$$(u \circ \Psi_j^{-1})|_{Z_j} =  u|_{X \cap U_j} = 0.$$
Define $h_{Z_j, B^k_{r_j}} = \sup_{z \in B^k_{r_j}}\inf_{z' \in Z_j}\|z-z'\|$.
We have established that there exists $C >0$, such that $\|u \circ \Psi_j^{-1}\|_{W^q_2(B^k_{r_j})} \leq C^q q^{\frac 32 q}\|u\|_\hh$, and by construction $(u \circ \Psi_j)|_{Z_j} = 0$. We can therefore apply Theorem~3.5 of \citet{rieger2010sampling} to obtain that there exists a $c_j, h_j > 0$, for which, when $h_{Z_j, B^k_{r_j}} \leq h_j$, then
$$\|u \circ \Psi_j^{-1}\|_{L^\infty(B^k_{r_j})} \leq \exp\left(-c_j h_{Z_j, B^k_{r_j}}^{-2/5}\right)\|u\|_\hh.$$
Now, denote by $\bar{h}_{S,U} = \sup_{x' \in U} \inf_{x \in S} d(x,x')$ with $d$ the geodesic distance over the manifold $\Omega$. By applying Theorem~8 of \citet{fuselier2012scattered}, we have that there exist $C$ and $h_0$ not depending on $X$ or $n$ such that, when $\bar{h}_{X,\Omega} \leq h_0$, the inequality 
$\bar{h}_{X_j \cap U_j, U_j} \leq C \bar{h}_{X,\Omega}$ holds for any $j \in [T]$.
Moreover, since by Theorem~6 of the same paper $\|x-x'\| \leq d(x,x') \leq C_1 \|x-x'\|$, for $C_1 > 1$ and $x,x' \in \Omega$, then  
$$h_{Z_j, B^k_{r_j}} \leq \bar{h}_{X_j \cap U_j, U_j} \leq C \bar{h}_{X,\Omega} \leq C C_1 h_{X,\Omega}.$$
Finally, defining $c_1 = c(2\max_j C_j)^{-2/5}, h = C_1^{-1} \min(h_0, C^{-1} \min_j h_j)$, when $h_{X,\Omega} \leq h$,
$$\|f - \widehat{f}_X\|_{L^\infty(\Omega)} = \max_{j \in [T]} \|u \circ \Psi_j^{-1}\|_{L^\infty(B^k_{r_j})} \leq e^{-c_1 h_{X,\Omega}^{-2/5}}\|f\|_\hh, \qquad \forall f \in \hh, ~~x \in \Omega.$$

The proof of Points~2 and~3 now proceeds as in \cref{thm:eig-ball}. Details are deferred to Appendix~\ref{sec:nystrom-app}.
\end{proof}
Point 1 of the result above is new, to our knowledge, and extends interpolation results on manifolds \citep{wendland2004scattered,fuselier2012scattered,hangelbroek2010kernel}, from polynomial to exponential decay, generalizing a technique of \citet{rieger2010sampling} to a subset of real analytic manifolds.
Points 2 and 3 are a generalization of \cref{thm:eig-ball} to the case of manifolds. In particular, the crucial point is that now the eigenvalue decay and the effective dimension depend on the dimension of the manifold $k$ and not the ambient dimension $d \gg k$. We think that the factor $5/2$ in the exponent of the eigenvalues and effective dimension is a result of the specific proof technique used and could be removed with a refined analysis, which is out of the scope of this paper.

We finally conclude the desired bound on the optimal rank in the manifold case.

\begin{cor}\label{thm:nystrom-manifold}
Let $\eps' \in (0, 1)$, $\eta > 0$, and let $\Omega \subset \R^d$ be a manifold of dimensionality $k \leq d$ satisfying \cref{asm:good-manifold}.
There exists $c_{\Omega, \eta} > 0$ not depending on $X$ or $n$ such that
\begin{equation*}
r^*(X, \eta, \eps') \leq c_{\Omega, \eta} \left(\log \frac{n}{\eps'}\right)^{5k/2}
\end{equation*}
\end{cor}
\begin{proof}
By the definition of $r^*(X, \eta, \eps')$ and the bound of \cref{thm:eig-manifold}, we have
\begin{equation*}
r^*(X, \eta, \eps') \leq \left(c_1(4 \eta R^2 + \log \tfrac{2n}{\eps'})\right)^{5k/2} + c_2\,.
\end{equation*}
Since $\log \tfrac{2n}{\eps'} \geq 1$, we may set $c_{\Omega, \eta} = \max\left\{(8 c_1 \eta R^2)^{5k/2}, c_2\right\}$ to obtain the claim.
\end{proof}
\section{Sinkhorn scaling an approximate kernel matrix}\label{sec:sink}

The main result of this section, presented next, gives both a runtime bound and an error bound on the approximate Sinkhorn scaling performed in line~\ref{line:sinkhorn} of Algorithm~\ref{alg:main}.\footnote{Pseudocode for the variant we employ can be found in Appendix~\ref{app:sink-code}.} The runtime bound shows that we only need a small number of iterations to perform this approximate Sinkhorn projection on the approximate kernel matrix. The error bound shows that the objective function $V_C(\cdot)$ in~\eqref{eq:peta} is stable with respect to both (i) Sinkhorn projecting an \emph{approximate kernel matrix} $\tilde{K}$ instead of the true kernel matrix $K$, and (ii) only performing an \emph{approximate Sinkhorn projection}.

The results of this section apply to any bounded cost matrix $C \in \RR^{n \times n}$, not just the cost matrix $C_{ij} = \|x_i - x_j\|_2^2$ for the squared Euclidean distance. To emphasize this, we state this result and the rest of this section in terms of an arbitrary such matrix $C$. Note that $\|C\|_{\infty} \leq 4R^2$ when $C_{ij} = \|x_i - x_j\|_2^2$ and all points lie in a Euclidean ball of radius $R$. We therefore state all results in this section for $\eps' := \min(1, \tfrac{\eps \eta}{50(\|C\|_{\infty}\eta + \log \tfrac{n}{\eta \eps})})$.

\begin{theorem}\label{thm:line_sinkhorn}
If $K = e^{- \eta C}$ and if $\tilde K \in \RR_{> 0}^{n \times n}$ satisfies $\|\log K - \log\tilde  K\|_{\infty} \leq \eps'$,
then Line~\ref{line:sinkhorn} of Algorithm~\ref{alg:main} outputs $D_1$, $D_2$, and $\hat W$ such that $\tilde P \defeq D_1 \tilde K D_2$ satisfies $\|\tilde P \bone - \p \|_1 + \|\tilde P^\top \bone - \q\|_1 \leq \eps'$ and
\begin{subequations}
\begin{align}
|\VC(P^\eta) - \VC(\tilde P)| & \leq \frac{\eps}{2} \label{eq:obj_bound} \\
|\hat W - \VC(\tilde P)| & \leq \frac{\eps}{2} \label{eq:cost_bound}
\end{align}
\end{subequations}
Moreover, if matrix-vector products can be computed with $\tilde{K}$ and $\tilde{K}^\top$ in time $\Tmult$, then this takes time $\Otilde((n+ \Tmult) \eta \Cinf \ep'^{-1})$.
\end{theorem}

The running time bound in Theorem~\ref{thm:line_sinkhorn} for the time required to produce $D_1$ and $D_2$ follows directly from prior work which has shown that Sinkhorn scaling can produce an approximation to the Sinkhorn projection of a positive matrix in time nearly independent of the dimension $n$.

\begin{theorem}[\citealp{AltWeeRig17,DvuGasKro18}]\label{thm:sink_run}
Given a matrix $\tilde K \in \RR_{> 0}^{n \times n}$, the Sinkhorn scaling algorithm computes diagonal matrices $D_1$ and $D_2$ such that $\tilde P \defeq D_1 \tilde K D_2$ satisfies $\|\tilde P \bone - \p\|_1 + \|\tilde P^\top \bone - \q\|_1 \leq \delta$ in $O(\delta^{-1} \log \tfrac{n}{\delta \min_{ij} \tilde K_{ij}})$ iterations, each of which requires $O(1)$ matrix-vector products with $\tilde{K}$ and $O(n)$ additional processing time.
\end{theorem}

Lemma~\ref{lem:sink_cost} establishes that computing the approximate cost $\hat W$ requires $O(n + \Tmult)$ additional time.
To obtain the running time claimed in Theorem~\ref{thm:line_sinkhorn}, it therefore suffices to use the fact that $\log \tfrac{1}{\min_{ij} \tilde{K}_{ij}} \leq \log \tfrac{1}{\min_{ij} K_{ij}} + \|\log K - \log \tilde{K}\|_{\infty} \leq \eta \|C\|_{\infty} + \eps'$.

The remainder of the section is devoted to proving the error bounds in Theorem~\ref{thm:line_sinkhorn}.
Subsection~\ref{subsec:sink:lip} proves stability bounds for using an approximate kernel matrix, Subsection~\ref{subsec:sink:marg} proves stability bounds for using an approximate Sinkhorn projection, and then Subsection~\ref{subsec:sink:proof} combines these results to prove the error bounds in Theorem~\ref{thm:line_sinkhorn}.

\subsection{Using an approximate kernel matrix}\label{subsec:sink:lip}
Here we present the first ingredient for the proof of Theorem~\ref{thm:line_sinkhorn}: that Sinkhorn projection is Lipschitz with respect to the logarithm of the matrix to be scaled. If we view Sinkhorn projection as a saddle-point approximation to a Gibbs distribution over the vertices of $\Coup$~\citep[see discussion by][]{KosYui94}, then this result is analogous to the fact that the total variation between Gibbs distributions is controlled by the $\ell_\infty$ distance between the energy functions~\citep{Sim79}.
\begin{prop}\label{prop:sink-perturbation-mat}
For any $\p, \q \in \Delta_n$ and any $K, \tilde K \in \RR_+^{n \times n}$,
\begin{equation*}
\|\Sink(K) - \Sink(\tilde K)\|_1 \leq \|\log K - \log \tilde K\|_\infty\,.
\end{equation*}
\end{prop}
\begin{proof}
Note that $-H(P)$ is $1$-strongly convex with respect to the $\ell_1$ norm~\citep[Section~4.3]{Bub15}.
By Proposition~\ref{prop:kkt}, $\Sink(K) = \argmin_{P \in \Coup} \langle -\log K, P \rangle - H(P)$ and $\Sink(\tilde K) = \argmin_{P \in \Coup} \langle - \log \tilde K, P \rangle - H(P)$.
The claim follows upon applying Lemma~\ref{lem:sc-perturbation}.
\end{proof}

In words, Proposition~\ref{prop:sink-perturbation-mat} establishes that the Sinkhorn projection operator is Lipschitz on the ``logarithmic scale.''
By contrast, we show in Appendix~\ref{app:sink-stable} that the Sinkhorn projection does not satisfy a Lipschitz property in the standard sense for any choice of matrix norm.

%

\subsection{Using an approximate Sinkhorn projection}\label{subsec:sink:marg}
Here we present the second ingredient for the proof of Theorem~\ref{thm:line_sinkhorn}: that the objective function $V_C(\cdot)$ for Sinkhorn distances in~\eqref{eq:peta} is stable with respect to the target row and column sums $\p$ and $\q$ of the outputted matrix.

\begin{prop}\label{prop:sum-stability}
Given $\tilde K \in \RR_{> 0}^{n \times n}$, let $\tilde C \in \RR^{n \times n}$ satisfy $\tilde C_{ij} \defeq - \eta^{-1} \log \tilde K_{ij}$.
Let $D_1$ and $D_2$ be positive diagonal matrices such that $\tilde P \defeq D_1 \tilde K D_2 \in \Delta_{n \times n}$, with $\delta \defeq \|\p - \tilde P \bone\|_1 + \|\q - \tilde P^\top \bone\|_1$. If $\delta \leq 1$, then
\begin{equation*}
|\VCtilde(\Sink(\tilde K)) - \VCtilde(\tilde P)| \leq \delta \Ctildeinf + \eta^{-1} \delta \log \frac{2n}{\delta}\,,
\end{equation*}
\end{prop}
\begin{proof}
Write $\tilde \p \defeq \tilde P \bone$ and $\tilde \q \defeq \tilde P^\top \bone$.
Then $\tilde P = \Sinkcouptilde(\tilde K)$ by the definition of the Sinkhorn projection.
If we write $P^* \defeq \Sinkcoup(\tilde K)$, then
Proposition~\ref{prop:kkt} implies
\begin{align*}
\tilde P & = \argmin_{P \in \Couptilde} \VCtilde(P) \\
P^* & = \argmin_{P \in \Coup} \VCtilde(P)
\end{align*}
Lemma~\ref{lem:hausdorff} establishes that the Hausdorff distance between $\Couptilde$ and $\Coup$ with respect to $\|\cdot\|_1$ is at most $\delta$, and by Lemma~\ref{lem:L1}, the function $\VCtilde$ satisfies
\begin{equation*}
|\VCtilde(P) - \VCtilde(Q)| \leq \omega(\|P - Q\|_1)\,,
\end{equation*}
where $\omega(\delta) \defeq \delta \|\Ctilde\|_\infty + \eta^{-1} \delta \log \frac{2n}{\delta}$ is increasing and and continuous on $[0, 1]$ as long as $n \geq 2$.
Applying Lemma~\ref{lem:round} then yields the claim.
\end{proof}

\subsection{Proof of Theorem~\ref{thm:line_sinkhorn}}\label{subsec:sink:proof}

The runtime claim was proven in Section~\ref{sec:sink}; here we prove the error bounds. We first show~\eqref{eq:obj_bound}. Define $\tilde C \defeq - \eta^{-1} \log \tilde K$. Since $\Peta = \Sink(K)$ by Corollary~\ref{cor:sink-basic}, we can decompose the error as
\begin{subequations}
\label{eq:err}
\begin{align}
|\VC(\Peta) - \VC(\tilde P)|
& \leq
\abs{\VC\left(\Sink\left(K \right)\right) - \VC\left(\Sink\left(\Ktilde \right)\right)}
\label{eq:err-nystrom}
\\ &+
\abs{\VC\left(\Sink\left(\Ktilde \right)\right) - \VCtilde\left(\Sink\left(\Ktilde \right)\right)}
\label{eq:err-cost-1}
\\ &+
\abs{\VCtilde\left(\Sink\left(\Ktilde \right)\right) - \VCtilde\left(\tilde P\right)}
\label{eq:err-sink}
\\ &+
\abs{\VCtilde\left(\tilde P \right) - \VC\left(\tilde P \right)}.
\label{eq:err-cost-2}
\end{align}
By Proposition~\ref{prop:sink-perturbation-mat} and Lemma~\ref{lem:L1}, term~\eqref{eq:err-nystrom} is at most $\eps' \Cinf + \eta^{-1} \eps' \log \frac{2n}{\eps'}$.
Proposition~\ref{prop:sum-stability} implies that~\eqref{eq:err-sink} is at most $\eps' \Ctildeinf + \eta^{-1} \eps' \log \frac{2n}{\eps'}$.
Finally, by Lemma~\ref{lem:cost-perturb}, terms~\eqref{eq:err-cost-1} and~\eqref{eq:err-cost-2} are each at most $\eta^{-1} \eps'$. Thus
\begin{align*}
|\VC(\Sink(K)) - \VC(\tilde P)|
& \leq
\left(\eps' \Cinf + \eta^{-1} \eps' \log \frac{2n}{\eps'}\right) + 
\left(\eps' \Ctildeinf + \eta^{-1} \eps' \log \frac{2n}{\eps'}\right) +
2\eta^{-1} \eps' \\
& \leq
2\eps' \Cinf + \eta^{-1} (\eps'^2 + 2 \eps') + 2\eta^{-1} \eps' \log \frac{2n}{\eps'} \\
& \leq \eps'(2 \Cinf + 3 \eta^{-1}) + 2\eta^{-1} \eps' \log \frac{2n}{\eps'}\,,
\end{align*}
where the second inequality follows from the fact that $\Ctildeinf \leq \Cinf + \|C - \tilde C\|_\infty \leq \Cinf + \eta^{-1} \eps'$. The proof of \eqref{eq:obj_bound} is then complete by invoking Lemma~\ref{lem:annoying}.

To prove \eqref{eq:cost_bound}, by \cref{lem:sink_cost} we have $\hat W = V_{\tilde C}(\tilde P)$, and by Lemma~\ref{lem:cost-perturb}, we therefore have $|\hat W - V_C(\tilde P)| \leq \eta^{-1} \eps' \leq \tfrac{\eps}{2}$.
\end{subequations}

\section{Proof of Theorem~\ref{thm:main}}\label{sec:alg}
In this section, we combine the results of the preceding three sections to prove Theorem~\ref{thm:main}. 

\paragraph*{Error analysis.}
First, we show that
\begin{align}
\abs{V_C(\Phat) - \Obj} = \abs{ \VC(\hat{P}) - \VC(\Peta) } \leq \eps \,.
\label{eq:pf-main:err}
\end{align}
We do so by bounding $| \VC(\Phat) - \VC(\tilde{P})| + |\VC(\tilde{P}) - \VC(\Peta)|$, where $\tilde P \defeq D_1 \Kt D_2$ is the approximate projection computed in \cref{line:sinkhorn}.
By \cref{lem:nystrom-doubling}, the output of \cref{line:nystrom} satisfies 
$\|K - \Kt\|_{\infty} \leq \tfrac{\eps'}{2}e^{-4 \eta R^2}$,
and by \cref{lem:trace-to-log} this implies that $\|\log K - \log \Kt\|_{\infty} \leq \eps'$.
Therefore, by Theorem~\ref{thm:line_sinkhorn}, $|\VC(\tilde{P}) - \VC(\Peta)| \leq \tfrac{\eps}{2}$.
Moreover, by Lemma~\ref{lem:round-alg}, $\|\tilde{P} - \Phat\|_1 \leq \|\tilde{P}\bone - \p \|_1 + \|\tilde{P}^T \bone - \q \|_1 \leq \eps'$, thus by an application of Lemmas~\ref{lem:L1} and~\ref{lem:annoying}, we have that $| \VC(\Phat) - \VC(\tilde{P})| \leq \tfrac{\eps}{2}$.
Therefore $| \VC(\Phat) - \VC(\tilde{P})| + |\VC(\tilde{P}) - \VC(\Peta)| \leq \eps$, which proves~\eqref{eq:pf-main:err} and thus also~\eqref{eq:thm-main:close-obj}.

Next, we prove~\eqref{eq:thm-main:close-kl}. By Proposition~\ref{prop:kkt}, $\Peta = \argmin_{P \in \Coup} \VC(P)$. Thus 
\[
\eps \geq \VC(\Phat) - \VC(\Peta) = \nabla \VC(\Peta)(\Phat - \Peta)  + \eta^{-1}\KL{\Phat}{\Peta}
\geq \eta^{-1}\KL{\Phat}{\Peta}
\,.
\]
where above the first inequality is by~\eqref{eq:pf-main:err}, the equality is by Lemma~\ref{lem:fo-kl}, and the final inequality is by first-order KKT conditions which give $\nabla \VC(\Peta)(\Phat - \Peta) \geq 0$. After rearranging, we conclude that $\KL{\Phat}{\Peta} \leq \eta \eps$, proving~\eqref{eq:thm-main:close-kl}.

Finally, by Theorem~\ref{thm:line_sinkhorn}, $|\hat W - V_C(\tilde P)| \leq \tfrac{\eps}{2}$, and we have already shown in our proof of \eqref{eq:thm-main:close-obj} that $|\VC(\tilde P) - \VC(\Peta)| \leq \tfrac{\eps}{2}$,
which proves~\eqref{eq:thm-main:approx-val}.

\paragraph*{Runtime analysis.}
Let $r$ denote the rank of $\tilde{K}$. Note that $r$ is a random variable. By Lemma~\ref{lem:nystrom-doubling}, we have that
\begin{align}
\PP\big(
r \leq c r^*(X, \eta, \eps') \log \tfrac{n}{\delta} 
\big)
\geq 1 - \delta.
\label{eq:pf-main:r}
\end{align}
Now by Lemma~\ref{lem:nystrom-doubling}, the $\AdaptiveNsytrom$ algorithm in line~\ref{line:nystrom} runs in time $O(nr^2)$, and moreover further matrix-vector multiplications with $\tilde{K}$ can be computed in time $O(nr)$. Thus the $\Sinkhorn$ algorithm in line~\ref{line:sinkhorn} runs in time $\tilde O(nr\eta R^2 \eps'^{-1})$ by Theorem~\ref{thm:line_sinkhorn},
and the $\Round$ algorithm in line~\ref{line:round} runs in time $O(nr)$ by Lemma~\ref{lem:round-alg}. Combining these bounds and using the choice of $\eps'$ completes the proof of Theorem~\ref{thm:main}.
\qed
\section{Experimental results}\label{sec:experimental}

In this section we empirically validate our theoretical results. 
To run our experiments, we used a desktop with 32GB ram and 16 cores Xeon E5-2623 3GHz. The code is optimized in terms of matrix-matrix and matrix-vector products using BLAS-LAPACK primitives.

\begin{figure*}[t]
	\centering
	\includegraphics[trim={5cm 0.5cm 3.5cm 0.5cm},width=\linewidth,clip]{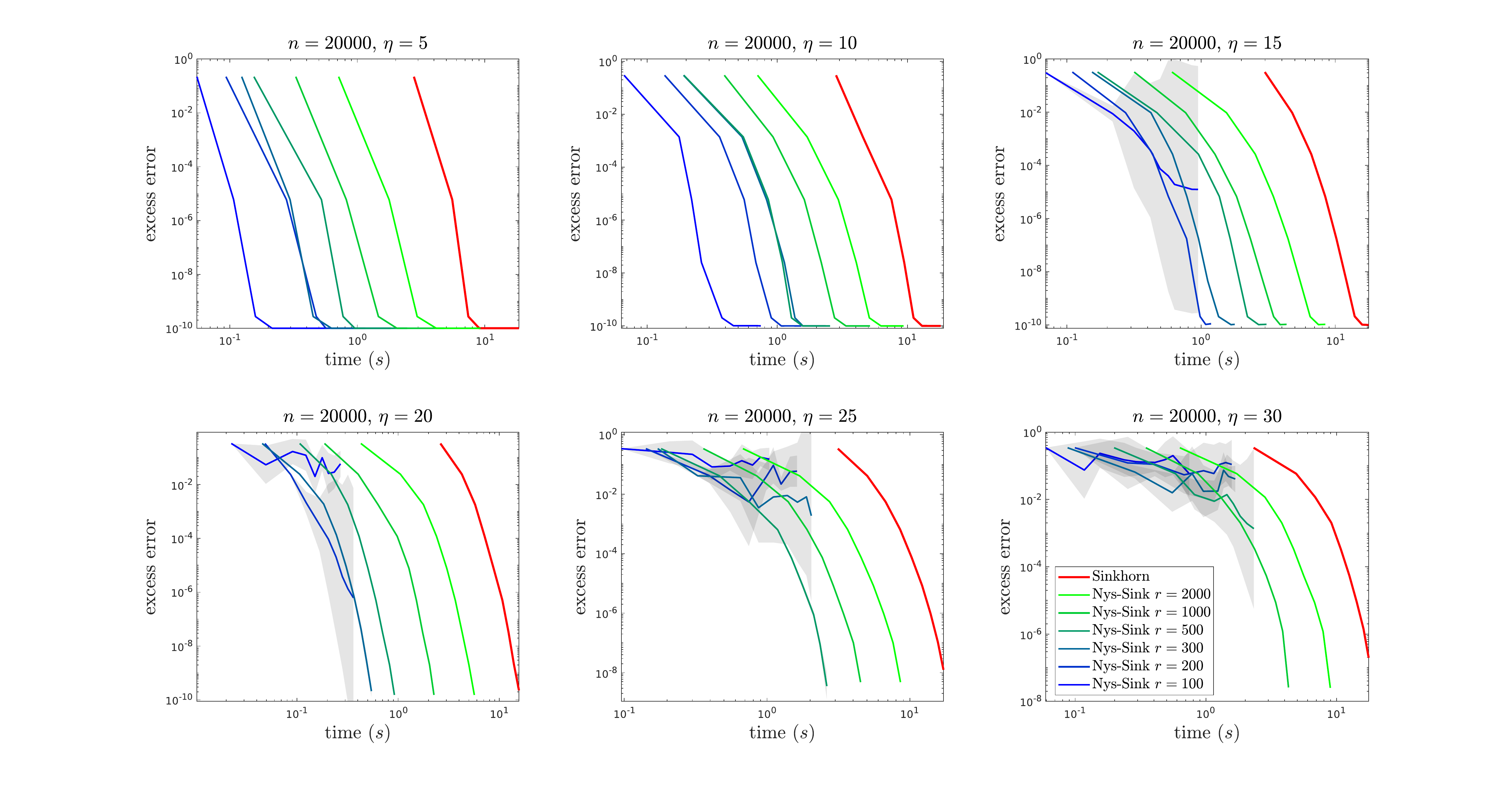}
	\caption{Time-accuracy tradeoff for \textsc{Nys-Sink} and \textsc{Sinkhorn}, for a range of regularization parameters $\eta$ (each corresponding to a different Sinkhorn distance $W_{\eta}$) and approximation ranks $r$. Each experiment has been repeated $50$ times; the variance is indicated by the shaded area around the curves. Note that curves in the plot start at different points corresponding to the time required for initialization.}
	\label{fig:err-time}
\end{figure*}

\cref{fig:err-time} plots the time-accuracy tradeoff for \textsc{Nys-Sink}, compared to the standard \textsc{Sinkhorn} algorithm. This experiment is run on random point clouds of size $n \approx 20000$, which corresponds to cost matrices of dimension approximately $20000 \times 20000$.~\cref{fig:err-time} shows that \emph{\textsc{Nys-Sink} is consistently orders of magnitude faster to obtain the same accuracy}.

\begin{figure*}[t]
	\centering
	\includegraphics[scale=0.4]{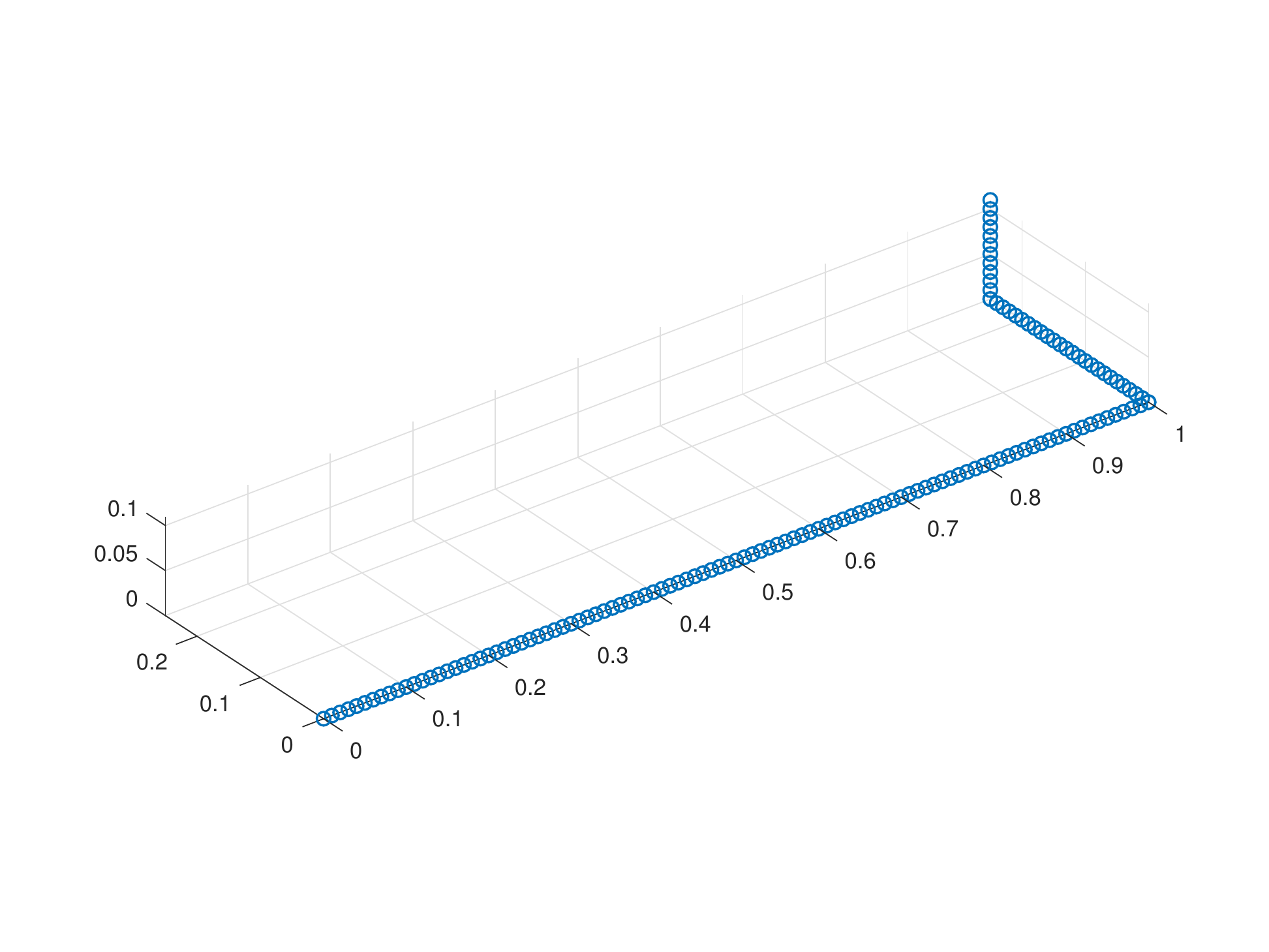}
	\includegraphics[scale=0.4]{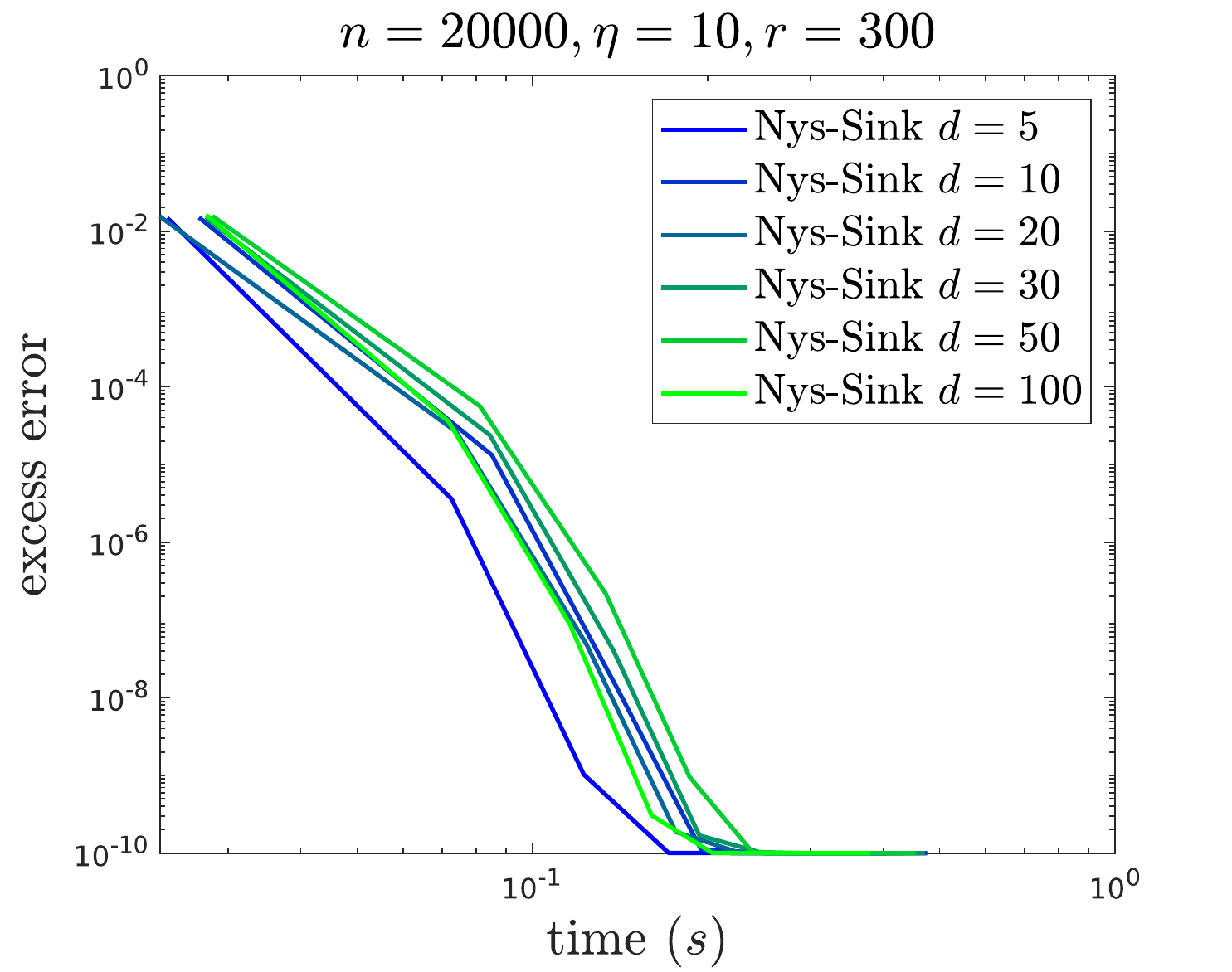}
	\caption{Left: one-dimensional curve embedded in $\RR^d$, for $d=3$. For $d \geq 4$, the curve we use in dimension $d$ is obtained from the curve we use in the dimension $d-1$ by adding a perpendicular segment of length $1/d^2$ to one endpoint. Right: Accuracy of \textsc{Nys-Sink} as a function of running time, for different ambient dimensions. Each experiment uses a fixed approximation rank $r = 300$.}
	\label{fig:ambient-dim}
\end{figure*}

Next, we investigate \textsc{Nys-Sink}'s dependence on the intrinsic dimension and ambient dimension of the input. This is done by running \textsc{Nys-Sink} on distributions supported on $1$-dimensional curves embedded in higher dimensions, illustrated in~\cref{fig:ambient-dim}, left.~\cref{fig:ambient-dim}, right, indicates that an approximation rank of $r=300$ is sufficient to achieve an error smaller than $10^{-4}$ for any ambient dimension $5 \leq d \leq 100$. This empirically validates the result in \cref{thm:nystrom-manifold}, namely that \emph{the approximation rank -- and consequently the computational complexity of \textsc{Nys-Sink} -- is independent of the ambient dimension.}

\begin{table}[t]
	\centering
	\includegraphics[width=0.5\textwidth]{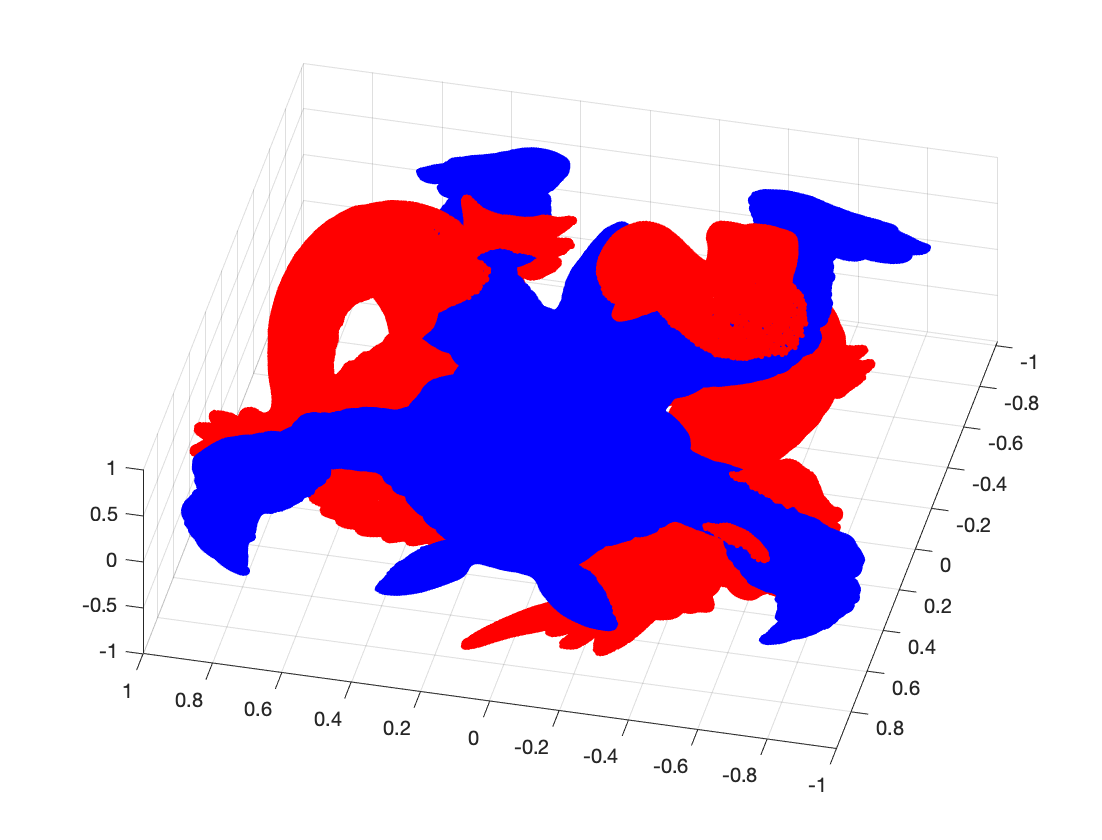}
	\begin{tabular}{l|c c}
		\toprule 
		{Experiment~1: $n \approx 3 \times 10^5$}  & $W_\eta$ & time (s) \\ 
		\midrule 
		Nys-Sink $(r=2000, T=20)$ & 0.087 $\pm$ 0.008 & 0.4 $\pm$ 0.1  \\
		Dual-Sink + Annealing $(\alpha=0.95)$ & 0.087 & 35.4  \\
		Dual-Sink Multiscale + Annealing $(\alpha=0.95)$ & 0.090 & 3.4\\
		\bottomrule
	\end{tabular}
	\begin{tabular}{l|c c}
		\toprule 
		{Experiment~2: $n \approx 3.8 \times 10^6$}  & $W_\eta$ & time (s) \\ 
		\midrule 
		Nys-Sink $(r=2000, T=20)$ & 0.11 $\pm$ 0.01 & 6.3 $\pm$ 0.8  \\
		Dual-Sink + Annealing $(\alpha=0.95)$ & 0.10 & 1168  \\
		Dual-Sink Multiscale + Annealing $(\alpha=0.95)$ & 0.11 & 103.6 \\
		\bottomrule
	\end{tabular}
	\caption{Comparison of our proposed algorithm to existing, highly-optimized GPU-based algorithms, on a large-scale computer graphics benchmark dataset.}
	\label{tab:graphics}
\end{table}

Finally, we evaluate the performance of our algorithm on a benchmark dataset used in computer graphics: we measure Wasserstein distance between 3D cloud points from ``The Stanford 3D Scanning Repository''\footnote{\url{http://graphics.stanford.edu/data/3Dscanrep/}}. In the first experiment, we measure the distance between \texttt{armadillo} ($n=1.7\times 10^5$ points) and \texttt{dragon} (at resolution 2, $n=1.0\times 10^5$ points), and in the second experiment we measure the distance between \texttt{armadillo} and \texttt{xyz-dragon} which has more points ($n=3.6 \times 10^6$ points). The point clouds are centered and normalized in the unit cube. The regularization parameter is set to $\eta = 15$, reflecting the moderate regularization regime typically used in practice.

We compare our algorithm (Nys-Sink)---run with approximation rank $r=2000$ for $T=20$ iterations on a GPU---against two algorithms implemented in the library \texttt{GeomLoss}\footnote{\url{http://www.kernel-operations.io/geomloss/}}. These algorithms are both highly optimized and implemented for GPUs. They are: (a) an algorithm based on an annealing heuristic for $\eta$ (controlled by the parameter $\alpha$, such that at each iteration $\eta_t = \alpha \eta_{t-1}$, see \citealp{kosowsky1994invisible}) and (b) a multiresolution algorithm based on coarse-to-fine clustering of the dataset together with the annealing heuristic~\citep{schmitzer2019stabilized}. Table~\ref{tab:graphics} reports the results, which demonstrate that our method is comparable in terms of precision, and has computational time that is orders of magnitude smaller than the competitors. We note the parameters $r$ and $T$ for \textsc{Nys-Sink} are chosen by hand to balance precision and time complexity.

We note that in these experiments, instead of using Algorithm~\ref{alg:adaptive-nystrom} to choose the rank adaptively, we simply run experiments with a small fixed choice of $r$. As our experiments demonstrate, \textsc{Nys-Sink} achieves good empirical performance even when the rank $r$ is smaller than our theoretical analysis requires. Investigating this empirical success further is an interesting topic for future study.

\newpage
\appendix
\section{Pseudocode for subroutines}

\subsection{Pseudocode for Sinkhorn algorithm}\label{app:sink-code}
As mentioned in the main text, we use the following variant of the classical Sinkhorn algorithm for our theoretical results. Note that in this paper, $\tilde{K}$ is not stored explicitly but instead is stored in factored form. This enables the Sinkhorn algorithm to be implemented quickly since all computations using $\tilde{K}$ are matrix-vector multiplications with $\tilde{K}$ and $\tilde{K}^T$ (see discussion in~\ref{subsec:nystrom:overview} for details).

\begin{algorithm}[H]
\begin{algorithmic}[1]
\Require{$\tilde K$ (in factored form), $\p, \q \in \Delta_n$, $\delta > 0$}
\Ensure{Positive diagonal matrices $D_1, D_2 \in \RR^{n \times n}$, cost $\hat W$}
\State $\tau \gets \tfrac{\delta}{8}$, $D_1, D_2 \gets I_{n \times n}$, $k \gets 0$
\State $\p' \gets (1-\tau)\p + \frac{\tau}{n} \bone$, $\q' \gets (1-\tau)\q + \frac{\tau}{n} \bone$  \Comment Round $\p$ and $\q$
\While{$\|D_1 \tilde K D_2 \bone - \p'\|_1 + \|(D_1 \tilde K D_2)^\top \bone - \q'\|_1 \leq \tfrac{\delta}{2}$}
\State $k \gets k + 1$
\If{$k$ odd}
\State $(D_1)_{ii} \gets \p'_i/(\tilde K D_2 \bone)_i$ for $i = 1, \dots, n$.  \Comment Renormalize rows
\Else
\State $(D_2)_{jj} \gets \q'_j/((D_1 \tilde K)^\top \bone)_j$ for $j = 1, \dots, n$. \Comment Renormalize columns
\EndIf
\EndWhile
\State $\hat W \gets \sum_{i=1}^n \log (D_1)_{ii} (D_1 \tilde K D_2)\bone)_i + \sum_{j=1}^n \log (D_2)_{jj} ((D_1 \tilde K D_2)^\top \bone)_j$
\State \Return{$D_1$, $D_2$, $\hat W$}
\end{algorithmic}
\caption{\textsc{Sinkhorn}}
\label{alg:sinkhorn}
\end{algorithm}

\begin{applemma}\label{lem:sink_cost}
Let $\tilde C := - \eta^{-1} \log \tilde K$ and let $\tilde P := D_1 \tilde K D_2$, where $D_1$ and $D_2$ are the scaling matrices output by \textsc{Sinkhorn}.
Then the output $\hat W$ of \textsc{Sinkhorn} satisfies
$\hat W = V_{\tilde C}(\tilde P)$.
Moreover, computing $\hat W$ takes time $O(\Tmult + n)$, where $\Tmult$ is the time required to take matrix-vector products with $\tilde K$ and $\tilde K^\top$.
\end{applemma}
\begin{proof}
Then
\begin{align*}
    \langle \tilde C, \tilde P \rangle -\eta^{-1} H(\tilde P) & = \langle \tilde C, \tilde P \rangle + \eta^{-1} \sum_{i,j = 1}^n \tilde P_{ij} \log \tilde P_{ij} \\
    & = \langle \tilde C, \tilde P \rangle + \eta^{-1} \sum_{i, j = 1}^n \tilde P_{ij} (\log (D_1)_{ii} + \log (D_2)_{jj} - \eta \tilde C_{ij}) \\
    & = \sum_{i, j=1}^n \tilde P_{ij} \log (D_1)_{ii} + \sum_{i, j=1}^n \tilde P_{ij} \log (D_2)_{jj} \\
    & = \sum_{i=1}^n \log(D_1)_{ii} (\tilde P \bone)_i + \sum_{j=1}^n \log(D_2)_{jj} (\tilde P^\top \bone)_j = \hat W\,.
\end{align*}
Moreover, the matrices $\log(D_1)$ and $\log(D_2)$ can each be formed in $O(n)$ time, so computing $\hat W$ takes time $O(\Tmult + n)$, as claimed.
\end{proof}

\subsection{Pseudocode for rounding algorithm}\label{app:round-code}
For completeness, here we briefly recall the rounding algorithm $\textsc{Round}$ from~\citep{AltWeeRig17} and prove a slight variant of their Lemma 7 that we need for our purposes.
\par It will be convenient to develop a little notation. For a vector $x \in \R^n$, $\diag(x)$ denotes the $n \times n$ diagonal matrix with diagonal entries $[\diag(x)]_{ii} = x_i$. For a matrix $A$, $r(A) := A\bone$ and $c(A) := A^T \bone$ denote the row and column marginals of $A$, respectively. We further denote $r_i(A) = [r(A)]_i$ and similarly $c_j(A) := [c(A)]_j$. 

\begin{algorithm}[H]
\begin{algorithmic}[1]
\Require{$F \in \RR^{n \times n}$ and $\p,\q \in \Delta_n$}
\Ensure{$G \in \Coup$}
\State $X \gets \diag(x)$, where $x_i := \tfrac{\p_i}{r_i(F)} \wedge 1$
\State $F' \gets XF$
\State $Y \gets \diag(y)$, where $y_j := \tfrac{\q_j}{c_j(F')} \wedge 1$
\State $F'' \gets F'Y$
\State $\errr \gets \p - r(F'')$, $\errc \gets \q - c(F'')$
\State Output $G \gets F'' + \errr \errc^T / \|\errr\|_1$
\end{algorithmic}
\caption{\textsc{Round} (from Algorithm 2 in~\citep{AltWeeRig17})}
\label{alg:round}
\end{algorithm}

\begin{applemma}\label{lem:round-alg}
If $\p,\q \in \Delta_n$ and $F \in \RR_{\geq 0}^{n \times n}$, then $\textsc{Round}(F, \p, \q)$ outputs a matrix $G \in \Coup$ of the form $G = D_1 F D_2 + uv^\top$ for positive diagonal matrices $D_1$ and $D_2$ satisfying
\[
\|G - F\|_1 \leq \left[ \|F\bone - \p\|_1 + \|F^T\bone - \q\|_1 \right].
\]
Moreover, the algorithm only uses $O(1)$ matrix-vector products with $F$ and $O(n)$ additional processing time.
\end{applemma}
\begin{proof}
The runtime claim is clear. Next, let $\Delta := \|F\|_1 - \|F''\|_1 = \sum_{i=1}^n (r_i(F) - \p_i)_+ + \sum_{j=1}^n (c_j(F') - \q_j)_+$ denote the amount of mass removed from $F$ to create $F''$. Observe that $\sum_{i=1}^n (r_i(F) - \p_i)_+ = \half \|r(F) - \p\|_1$. Since $F' \leq F$ entrywise, we also have $\sum_{j=1}^n (c_j(F') - \q_j)_+ \leq \sum_{j=1}^n (c_j(F) - \q_j)_+ = \half\|c(F) - \q\|_1$. Thus $\Delta \leq \half ( \|r(F) - \p\|_1 + \|c(F) - q\|_1)$. The proof is complete since $\|F - G\|_1 
\leq
\|F - F''\|_1 + \|F'' - G\|_1
=
2\Delta$.
\end{proof}

\section{Omitted proofs}\label{sec:proofs}

\subsection{Stability inequalities for Sinkhorn distances}

\begin{applemma}\label{lem:cost-perturb}
Let $C, \tilde C \in \RR^{n \times n}$.
If $P \in \Delta_{n \times n}$, then
\begin{equation*}
|\VC(P) - \VCtilde(P)| \leq \|C - \tilde C\|_\infty\,.
\end{equation*}
\end{applemma}
\begin{proof}
By H\"older's inequality, $|\VC(P) - \VCtilde(P)| = |\langle C - \tilde C, P \rangle| \leq \|C - \tilde C\|_\infty \|P\|_1 =  \|C - \tilde C\|_\infty$.
\end{proof}

\begin{applemma}\label{lem:ent_cont}
Let $P, Q \in \Delta_{n \times n}$. If $\|P - Q\|_1 \leq \delta \leq 1$, then
\begin{equation*}
|H(P) - H(Q)| \leq \delta \log \frac{2 n}{\delta}\,.
\end{equation*}
\end{applemma}
\begin{proof}
By~\citet[Theorem~6]{HoYeu10}, $|H(P) - H(Q)| \leq \frac{\delta}{2} \log(n^2-1) + h\big(\frac{\delta}{2}\big)$, where $h$ is the binary entropy function.
If $\delta \leq 1$, then $h(\frac{\delta}{2}) \leq \delta \log \frac{2}{\delta}$, which yields the claim.
\end{proof}

\begin{applemma}\label{lem:L1}
Let $M \in \RR^{n \times n}$, $\eta > 0$, and $P, Q \in \Delta_{n \times n}$. If $\|P - Q\|_1 \leq \delta \leq 1$, then
\begin{equation*}
\abs{V_M(P) - V_M(Q)} \leq \delta \|M\|_{\infty} + \eta^{-1} \delta \log \frac{2n}{\delta}\,.
\end{equation*}
\end{applemma}
\begin{proof}
By definition of $V_M(\cdot)$ and the triangle inequality, $|V_M(P) - V_M(Q)| \leq |\sum_{ij} (P_{ij} - Q_{ij}) M_{ij}| + \eta^{-1} |H(P) - H(\tilde{P})|$. By H\"older's inequality, the former term is upper bounded by $\|P - Q\|_1 \|M\|_{\infty} \leq \delta \|M\|_{\infty}$. By Lemma~\ref{lem:ent_cont}, the latter term above is upper bounded by $\eta^{-1}\delta \log \tfrac{2n}{\delta}$.
\end{proof}

\subsection{Bregman divergence of Sinkhorn distances}
The remainder in the first-order Taylor expansion of $\VC(\cdot)$ between any two joint distributions is exactly the KL-divergence between them.

\begin{applemma}\label{lem:fo-kl}
For any $C \in \RR^{n \times n}$, $\eta > 0$, and $P, Q \in \Delta^{n \times n}$,
\begin{equation*}
\VC(Q) = \VC(P) +  \langle \nabla \VC(P), (Q - P)\rangle  + \eta^{-1}\KL{Q}{P}.
\end{equation*}
\end{applemma}
\begin{proof}
Observing that $\nabla \VC(P)$ has $ij$th entry $C_{ij} + \eta^{-1}(1 + \log P_{ij})$, we expand the right hand side as $[\langle C, P \rangle + \eta^{-1} \sum_{ij} P_{ij} \log P_{ij}]
+ 
[ \langle C, Q - P \rangle + \eta^{-1} \sum_{ij} (Q_{ij} - P_{ij}) \log P_{ij} \rangle ]  + 
[ \eta^{-1} \sum_{ij} Q_{ij} \log \frac{Q_{ij}}{P_{ij}} ] = \langle C, Q \rangle + \eta^{-1} \sum_{ij} Q_{ij} \log Q_{ij} = \VC(Q)$.
\end{proof}

\subsection{Hausdorff distance between transport polytopes}


\begin{applemma}\label{lem:hausdorff}
Let $d_H$ denote the Hausdorff distance with respect to $\|\cdot\|_1$.
If $\p, \tilde \p, \q, \tilde \q \in \Delta_n$, then
\begin{equation*}
d_H(\Coup, \Couptilde) \leq \|\p - \tilde \p\|_1 + \|\q - \tilde \q\|_1\,.
\end{equation*}
\end{applemma}
\begin{proof}
Follows immediately from Lemma~\ref{lem:round-alg}.
\end{proof}

\begin{applemma}\label{lem:round}
Fix a norm $\| \cdot \|$ on $\cX$.
If $f: \cX \to \RR$ satisfies $|f(x) - f(y)| \leq \omega(\|x - y\|)$ for $\omega$ an increasing, upper semicontinuous function, then for any two sets $A, B \subseteq \cX$,
\begin{equation*}
\abs{\inf_{x \in A} f(x) - \inf_{x \in B} f(x)} \leq \omega(d_H(A, B))\,,
\end{equation*}
where $d_H(A,B)$ is the Hausdorff distance between $A$ and $B$ with respect to $\| \cdot \|$.
\end{applemma}
\begin{proof}
\begin{align*}
\inf_{x \in A} f(x) - \inf_{x \in B} f(x) &\leq \adjustlimits \sup_{y \in B} \inf_{x \in A} f(x) - f(y) \\
& \leq \adjustlimits \sup_{y \in B} \inf_{x \in A} \omega(\|x - y\|) \\
& \leq \omega( \adjustlimits \sup_{y \in B} \inf_{x \in A} \|x - y\|) \\
& \leq \omega(d_H(A, B))\,.
\end{align*}
Interchanging the role of $A$ and $B$ yields the claim.
\end{proof}

\subsection{Miscellaneous helpful lemmas}
\begin{applemma}\label{lem:sc-perturbation}
Let $X \subset \R^d$ be convex, and let $f : X \to \R$ be $1$-strongly-convex with respect to some norm $\|\cdot\|$. If $x_a^* = \argmin_{x \in \cX} \langle a, x \rangle + f(x)$ and $x^*_b = \argmin_{x \in \cX} \langle b, x \rangle + f(x)$, then

\begin{equation*}
\|x_a^* - x^*_b\| \leq \|a - b\|_*\,,
\end{equation*}
where $\|\cdot\|_{*}$ denotes the dual norm to $\|\cdot\|$.
\end{applemma}
\begin{proof}
This amounts to the well known fact~\citep[see, e.g.,][Theorem~4.2.1]{HirLem01} that the Legendre transform of a strongly convex function has Lipschitz gradients. We assume without loss of generality that $f = + \infty$ outside of $\cX$, so that $f$ can be extended to a function on all of $\RR^d$ and thus we can take the minima to be unconstrained.
The fact that $f(y) + \langle y, a \rangle \geq f(x^*_a) + \langle x^*_a, a \rangle$ for all $y$ implies that $- a \in \partial f(x^*_a)$, and likewise $-b \in \partial f(x^*_b)$. Thus by definition of strong convexity, we have
\begin{align*}
f(x^*_a) &\geq f(x^*_b) + \langle -b, x^*_a - x^*_b \rangle + \half \|x^*_a - x^*_b\|^2,
\\ f(x^*_b) &\geq f(x^*_a) + \langle -a, x^*_b - x^*_a \rangle + \half \|x^*_b - x^*_a\|^2,
\end{align*}
Adding these inequalities yields
\begin{equation*}
\langle b - a, x^*_a - x^*_b \rangle \geq  \|x^*_a - x^*_b\|^2\,,
\end{equation*}
which implies the claim via the definition of the dual norm.
\end{proof}

\begin{applemma}\label{lem:norm-equiv}
For any matrix $A \in \RR^{n \times n}$,
\begin{align*}
\|A\|_{\infty} \leq \|A\|_{\mathrm{op}} \leq n \|A\|_{\infty}
\end{align*}
\end{applemma}
\begin{proof}
By duality between the operator norm and the nuclear norm, $\|A\|_{\infty} = \max_{i,j \in [n]} |e_i^TAe_j| \leq \max_{i,j \in [n]} \|A\|_{\mathrm{op}} \|e_ie_j^T\|_{*} = \|A\|_{\mathrm{op}}$. This establishes the first inequality.

Next, for any $v \in \RR^n$ with unit norm $\|v\|_2 = 1$, note that $\|A v\|^2_2 = \sum_{i=1}^n \left(\sum_{j=1}^n A_{ij} v_j\right)^2 \leq \sum_{i=1}^n n \|A\|^2_\infty \sum_{j=1}^n v_j^2 = n^2 \|A\|^2_\infty$, proving the second inequality.
\end{proof}

\begin{applemma}\label{lem:log-trick}
For any $a, b > 0$,
\begin{equation*}
|\log a - \log b| \leq \frac{|a - b|}{\min\{a, b\}}\,.
\end{equation*}
\end{applemma}
\begin{proof}
Without loss of generality, assume $a \geq b$.
Then $\log a - \log b = \log \tfrac a b \leq \tfrac a b - 1 = \tfrac{a - b}{\min\{a, b\}}\,,$ as claimed.
\end{proof}

\begin{applemma}\label{lem:trace-to-log}
Let $\{x_1, \dots, x_n \} \subset \RR^d$ lie in an Euclidean ball of radius $R$, and let $\eta > 0$. Denote by $K \in \RR^{n \times n}$ the matrix with entries $K_{ij} := e^{-\eta \|x_i - x_j\|_2^2}$. If a matrix $\Kt \in \RR^{n \times n}$ satisfies $\|K - \Kt\|_{\infty} \leq \frac{\eps'}{2} e^{-4 \eta R^2}$ for some $\eps' \in (0, 1)$, then
\begin{equation*}
\|\log K - \log \Kt\|_\infty \leq \eps'\,.
\end{equation*}
\end{applemma}
\begin{proof}
Since $\|x_i\|_2 \leq R$ for all $i \in [n]$, the matrix $K$ satisfies $K_{ij} = e^{-\eta \|x_i - x_j\|_2^2} \geq e^{- 4 \eta R^2}$ for all $i, j \in [n]$.
Hence $\Kt_{ij} \geq \tfrac{\eps'}{2} e^{- 4 \eta R^2}$ for all $i, j \in [n]$ and thus by Lemma~\ref{lem:log-trick},
\begin{equation*}
|\log K_{ij} - \log \Kt_{ij}| \leq \frac{|K_{ij} - \Kt_{ij}|}{\min \{\Kt_{ij}, K_{ij}\}} \leq \ep'\,.
\end{equation*}
\end{proof}

\begin{applemma}\label{lem:annoying}
Let $n \in \mathbb{N}$, $\eps \in (0,1)$, $\Cinf \geq 1$, and $\eta \in [1,n]$. Then for any $\delta \leq \tfrac{\eta \eps}{50(\Cinf\eta + \log \tfrac{n}{\eta \eps})}$, the bound $\delta (2\Cinf + 3\eta^{-1}) + 2\eta^{-1}\delta \log \tfrac{2n}{\delta} \leq \tfrac{\eps}{2}$ holds.
\end{applemma}
\begin{proof}
We write
\begin{equation*}
\delta (2\Cinf + 3\eta^{-1}) + 2\eta^{-1}\delta \log \tfrac{2n}{\delta} = \delta (2\Cinf + 3\eta^{-1}) + 2\eta^{-1}\delta \log \tfrac{2n}{\eta \eps} + 2\eta^{-1} \delta \log \tfrac{\eta\eps }{ \delta}
\end{equation*}
and bound the three terms separately.
First, the assumptions imply that $\eta \eps \leq n$ and $2\Cinf + 3\eta^{-1} \leq 5\Cinf$. We therefore have
\begin{equation*}
\delta (2\Cinf + 3\eta^{-1}) \leq \frac{5 \Cinf \eta \eps}{50(\Cinf\eta + \log \tfrac{n}{\eta \eps})} \leq \frac{1}{10}\ep\,.
\end{equation*}
Since $\Cinf \eta \geq 1$, we likewise obtain
\begin{equation*}
2\eta^{-1} \delta \log \frac{2n}{\eta \eps} \leq \frac{2\eps \log \frac{2n}{\eta \eps}}{50(1 + \log \frac{n}{\eta\eps})} = \frac{2(\log 2 + \log \frac{n}{\eta \eps})}{50(1 + \log \frac{n}{\eta\eps})} \ep \leq \frac{1}{25} \ep\,.
\end{equation*}

Finally, the fact that $\tfrac{\eta^{-1}\delta}{\eps} \leq \tfrac{1}{50}$ and $x \log \frac 1 x \leq \tfrac{1}{10}$ for $x \leq \tfrac{1}{50}$ yields
\begin{equation*}
2\eta^{-1} \delta \log \tfrac{\eta \eps}{ \delta}= 2\left(\tfrac{\eta^{-1} \delta}{\eps} \log \tfrac{\eps}{\eta^{-1} \delta}\right) \eps  \leq \frac{1}{5} \eps\,.
\end{equation*}
\end{proof}

\subsection{Supplemental results for Section~\ref{sec:nystrom}}\label{sec:nystrom-app}

\subsubsection{Full proof of Theorem~\ref{thm:eig-ball}}

Define $\phi_\alpha(x) := (2\eta)^{\sum_{j=1}^d \alpha_j / 2} \prod_{j=1}^d [(\alpha_j!)^{-1/2} x_j^{\alpha_j} e^{-\eta x_j^2}]$, for $x \in \R^d$ and $\alpha \in (\N \cup \{0\})^d$, and define $\psi_T(x) := (\phi_\alpha(x))_{\alpha_1 + \dots + \alpha_d \leq T}$. Note that $\psi_T: \R^d \to \R^M$, with $M = \binom{d+T}{T}$.
	    By \citet[equation 11]{cotter2011explicit}, we have
	    $$\sup_{x,x' \in B_R^d} |k_\eta(x,x') - \psi_T(x)^\top \psi_T(x')| \leq \frac{(2 \eta R^2)^{T+1}}{(T+1)!} =: \eps(T).$$
	    Now denote by $\Psi_T \in \R^{M \times n}$ the matrix $\Psi_T := (\psi_T(x_1),\dots, \psi_T(x_n))$.
	    By Lemma~\ref{lem:norm-equiv}, we have 
	    $$\|K - \Psi_T^\top\Psi_T\|_{\op} \leq n \sup_{i,j} |k_\eta(x_i,x_j) - \psi_T(x_i)^\top \psi_T(x_j)| \leq n \eps(T).$$
By the Eckart-Young-Mirsky Theorem, we have
\[
\lambda_{M+1}
=
\inf_{\bar{K}_T \in \RR^{n \times n},\, \rank(\bar{K}_T) \leq M} \|K - \bar{K}_T \|_{\op}.
\]
Therefore by combining the above two displays, we conclude that
\[
\lambda_{M+1}
\leq
\|K - \Psi_T^\top\Psi_T \|_{\op}
\leq n \eps(T).
\]

\paragraph{Point 1.}
%
%
We recall that for any $d, q \in \N$,
		the inequality $\binom{d+q}{q} \leq e^d (1+q/d)^d$ holds.
		Therefore, given $t \geq (2e)^d$, choosing $T = \lfloor d t^{1/d}/(2e) \rfloor$ yields $\binom{T + d}{d} \leq e^d (1+T/d)^d \leq t$.
		We therefore have $\lambda_{t+1}
\leq n \eps(T)$ for this choice of $T$.
		Now, by Stirling's approximation of $(T+1)!$, we have that $\eps(T) \leq e^{-(T+1) \log {T + 1 \over 2 e \eta R^2}}$.
		If $T \geq 2 e \eta R^2$, then $(T+1) \log {T + 1 \over 2 e \eta R^2} \geq \tfrac{d t^{1/d}}{2e} \log \tfrac{d t^{1/d}}{4 e^2 \eta R^2}$, which yields the desired bound.
		On the other hand, when $T < 2 e \eta R^2$, we use the trivial bound $\la_t(K) \leq \tr(K) \leq n$. The claim follows.

\paragraph{Point 2.} We have that $\la_{M_T+1}(K) \leq n \eps(T)$, for $M_T = \binom{d+T}{T}$ and $T \in \N$. Since the eigenvalues are in decreasing order we have that $\la_{M_{T+1}+1} \leq \la_t(K) \leq \la_{M_T +1}(K)$ for $M_{T} +1 \leq t \leq M_{T+1} +1 $. Since $x/(x + \tau)$ is increasing in $x$, for $x \geq 0$, we have
		$$ \sum_{t = 1}^n \frac{\la_j(K)}{\la_j(K) + n\tau} \leq \sum_{T = 0}^\infty (M_{T+1} - M_{T}) \frac{\la_{M_T+1}(K)}{\la_{M_T+1}(K) + n\tau} \leq \sum_{T = 0}^\infty (M_{T+1} - M_{T}) \frac{\eps(T)}{\eps(T) + \tau}.$$
		Let $T_\tau$ be such that $\eps(T_\tau) \leq \tau$. We can then bound $\eps(T)/(\eps(T)+\tau)$ above by $1$ for $T \leq T_\tau - 1$ and by $\eps(T)/\tau$ for $T \geq T_\tau$, obtaining
		\eqals{
		\sum_{T = 0}^\infty (M_{T+1} - M_{T}) \frac{\eps(T)}{\eps(T) + \tau} &\leq \sum_{T = 0}^{T_\tau - 1 } (M_{T+1} - M_{T})  ~ + ~\sum_{T = T_\tau}^\infty (M_{T+1} - M_{T}) \frac{\eps(T)}{\tau}\\
		& = M_{T_\tau}  + \frac{1}{\tau}\sum_{T = T_\tau}^\infty (M_{T+1} - M_{T})\eps(T).
		}
		In particular, we can choose $T_\tau = d + 2e^2\eta R^2 + \log(1/\tau)$. Since $\log \frac{T_\tau}{2e \eta R^2} > 1$, for any $T \geq T_\tau$, then 
		$\eps(T) \leq e^{-T \log \frac{T}{2e \eta R^2}} \leq e^{-T}$. Moreover since $M_{T+1} - M_T = d M_T/(T+1)$, and $M_T \leq e^d(1+T/d)^d$, we have
		\eqals{
		\sum_{T = T_\tau}^\infty (M_{T+1} - M_{T})\eps(T) & \leq \frac{d}{T_\tau} \sum_{T = T_\tau}^\infty M_T e^{-T} \leq \frac{d e^d}{T_\tau} \sum_{T = T_\tau}^\infty \left(1+{T \over d}\right)^d e^{-T} \\
		& \leq \frac{de^d}{T_\tau} \int_{T_\tau}^\infty \left(1+{x \over d}\right)^d e^{-x} dx.
		}
		Finally, by changing variables, $x = u+T_\tau$ and $u = (d+T_\tau)z$,
		\eqals{
		\int_{T_\tau}^\infty \left(1+{x \over d}\right)^d e^{-x} dx &= \int_{0}^\infty \left(1 + {T_\tau \over d}+{u \over d}\right)^d e^{-u-T_\tau} du \\
		&= \left(1+ {T_\tau \over d}\right)^d e^{-T_\tau} \int_0^\infty \left(1+{u \over d+T_\tau}\right)^d e^{-u} du \\
		& = \left(1+ {T_\tau \over d}\right)^d e^{-T_\tau} (d+T_\tau) \int_0^\infty \left(1+z\right)^d e^{-(d+T_\tau)z} dz \\
		& = d^{-d} e^d \Gamma(d+1, d+T_\tau),
		}
		where for the last equality we used the characterization of the incomplete gamma function $\Gamma(a,z) = z^{-a}e^{-z} \int_{0}^\infty (1+t)^{a-1}e^{-zt}dt$ \citep[see Eq. 8.6.5 of][]{olver2010nist}.
		To complete the proof note that by \cref{lm:incomplete-gamma} we have $\Gamma(a,z) \leq z/(z-a) z^{a-1} e^{-z}$, for any $z > a > 0$. Since $\log(1/\tau) \geq 0$ for $\tau \in (0, 1]$, we have $(d+T_\tau)/(T_\tau - 1) \leq 2$ and $(d e^{-T_\tau})/(\tau T_\tau) \leq 1$, so
		\eqals{
		\deff(\tau) &\leq M_{T_\tau} + \frac{de^d}{\tau T_\tau} d^{-d}e^d \Gamma(d+1,d+T_\tau) \\
		& \leq e^d(1+T_\tau/d)^d \left(1 + \frac{de^{-T_\tau}}{\tau T_\tau}\frac{d + T_\tau}{T_\tau-1} \right) \\
		& \leq 3e^d \left(2+\frac{2e^2}{d}\eta R^2+ \frac{1}{d}\log \frac{1}{\tau}\right)^d.
		}
\qed
\subsubsection{Full proof of Theorem~\ref{thm:eig-manifold}}
The proof of Points 2 and 3 here is completely analogous to the proof of Points 1 and 2, respectively, in \cref{thm:eig-ball}.
\paragraph{Point 2.}
Let $\bar{X} = \{\bar{x}_1,\dots, \bar{x}_p\} \subset \Omega$
be a minimal $\eps$ net of $\Omega$. Since $\Omega$ is a smooth manifold of dimension $k$, then there exists $C_0 > 0$ for which $p  \leq C_0 \eps^{-k}$.
Now let $\bar{K} \in \R^{p \times p}$ be given by $\bar{K}_{i,j} = k_\eta(\bar{x}_i,\bar{x}_j)$ and define
$\Phi(x) = \bar{K}^{-1/2}v(x)$, with $v(x) = (k_\eta(x,\bar{x}_1), \dots, k_\eta(x,\bar{x}_p))$. Then when $f(x') := k_\eta(x',x)$, then
$\widehat{f}_{\bar{X}} = \Phi(x')^\top\Phi(x)$.
By applying Point 1 to $f(x) = k_\eta(x',x)$, we have
$$|k_\eta(x',x) - \Phi(x')\Phi(x)| \leq e^{-c \eps^{-2/5}}, \qquad \forall x,x' \in \Omega.$$
 If we let $B \in \R^{p \times n} = (\Phi(x_1), \dots \Phi(x_n))$, then
$$ \|K - B^\top B\|_{\op} \leq n \max_{ij} |k_\eta(x',x) - \Phi(x')\Phi(x)| \leq n e^{-c \eps^{-2/5}}.$$
Since $B$ is of rank $p$, the the Eckart-Young-Mirsky Theorem again implies
$\la_{p+1}(K) \leq n e^{-c \eps^{-2/5}}$.
We conclude by recalling that $\eps \leq (p/C_0)^{-1/k}$.

\paragraph{Point 3.}
Let $M_\tau$ be such that $\lambda_{M_\tau + 1} \leq n \tau$.
By Point 2, this holds if we take $M_\tau = (c_0 \log \tfrac 1\tau)^{5k/2}$ for a sufficiently large constant $c$.
By definition of $\deff(\tau)$ and the fact that $x/(x+\la) \leq \min(1, x/\la)$ for any $x \geq0, \la > 0$, we have
\eqals{
\deff(\tau) &= \sum_{j=1}^n \frac{\la_j(K)}{\la_j(K) + n\tau} =  \sum_{j=1}^{M_\tau+1} \frac{\la_j(K)}{\la_j(K) + n\tau} + \sum_{j=M_\tau+2}^n \frac{\la_j(K)}{\la_j(K) + n\tau}\\
& \leq M_\tau+1 + \frac{1}{\tau}\sum_{j=M_\tau+2}^n \frac{\la_j(K)}{n}
\leq (c_0 \log \tfrac 1\tau)^{5k/2} + 1 + \frac{1}{\tau}\sum_{j=M_\tau+1}^\infty e^{-c j^{\frac{2}{5k}}}\\
}
Denoting $\beta := \tfrac{2}{5k}$ for shorthand, we can upper bound the sum as follows: 
\begin{align*}
\sum_{j=M_{\tau} + 1}^{\infty} e^{-cj^{\tfrac{2}{5k}}}
    &\leq \int_{M_{\tau}}^{\infty} e^{-cx^{\beta}} dx
    \\ &= \frac{1}{\beta c^{1/\beta}} \int_{0}^{\infty}  u^{\frac{1-\beta}{\beta}} \mathds{1}(u \geq c M_{\tau}^{\beta}) e^{-u} du
    \\ &\leq \frac{1}{\beta c^{1/\beta}} \left( \int_0^{\infty} u^{2(\tfrac{1-\beta}{\beta})} e^{-u} du \right)^{1/2} \left( \int_{cM_{\tau}^{\beta}}^{\infty} e^{-u} du \right)^{1/2}
    \\ &= c_{k} e^{-\half cM_{\tau}^{\beta}} \leq c_k \tau,
\end{align*}
where above the second step was by the change of variables $u := cx^{\beta}$, the third step was by Cauchy-Schwartz with respect to the inner product $\langle f, g \rangle := \int_0^{\infty} f(u)g(u) e^{-u} du$, and the final line was for some constant $c_k$ only depending on $k$, whenever $c_0$ is taken to be at least $\tfrac{2}{c}$. This proves the claim. 
\qed
\subsubsection{Additional bounds}
\begin{applemma}\label{lm:bound-composition}
Let $A: B \to U$, be a smooth map, with $B \subseteq \R^d$, $U \subseteq \R^m$, $d, m \in \N$, such that there exists $Q > 0$, for which
$$\|D^\alpha A\|_{L^{\infty}(B)} \leq Q^{|\alpha|}, \quad \forall \alpha \in \N^d,$$
then, for $\nu \in \N_0^d$, $p \geq 1$,
$$\|D^\nu (f \circ A)\|_{L^p(B)} \leq (2|\nu|m Q)^{|\nu|} \max_{|\la| \leq |\nu|} \|(D^\la f) \circ A\|_{L^p(B)}.$$
\end{applemma}
\begin{proof}
First we study $D^\nu (f \circ A)$. Let $n := |\nu|$ and $A=(a_1,\dots,a_m)$ with $a_j:\R^d\to\R$. By the {\em multivariate Faa di Bruno formula} \citep{constantine1996multivariate}, we have that
$$ D^\nu (f \circ A) = \nu! \sum_{1 \leq |\la| \leq n} (D^\la f) \circ A \sum_{\tiny\begin{pmatrix}k_1\dots, k_n\\ l_1,\dots, l_n \end{pmatrix} \in p(\la, \nu)} ~~\prod_{j=1}^{n} \prod_{i=1}^m
\frac{[D^{l_j}a_i]^{[k_j]_i}}{[k_j]_i! l_j!},$$
where the set $p(\la,\nu)$ is defined in \citet[Eq.~2.4]{constantine1996multivariate}, with $l_1,\dots, l_n \in \N_0^d$, $k_1,\dots,k_n \in \N_0^m$ and satisfying $\sum_{j=1}^n |k_j| l_j = \nu$.
Now by assumption $\|D^{l_j}a_i\| \leq Q^{|l_j|}$ for $1 \leq i \leq m$. Then
$$\left\|\prod_{j=1}^{n} \prod_{i=1}^m
\frac{[D^{l_j}a_i]^{[k_j]_i}}{[k_j]_i! l_j!}\right\|_{L^\infty(B)} \leq Q^{\sum_{j=1}^n |l_j||k_j|} \prod_{j=1}^{n} \frac{1}{k_j! l_j!^{|k_j|}}.$$
Now note that by the properties of $l_j,k_j$, we have that $|\nu| = |\sum_j |k_j| l_j| = \sum_j |k_j||l_j|,$
then
$$ \|D^\nu (f \circ A)\|_{L^p(B)} \leq Q^{|\nu|} \max_{|\la| \leq |\nu|} \|(D^\la f) \circ A\|_{L^p(B)} ~~ \times ~~ \nu! \sum_{1 \leq |\la| \leq n} \sum_{\tiny\begin{matrix}k_1\dots, k_n\\ l_1,\dots, l_n \end{matrix} \in p(\la, \nu)} ~\prod_{j=1}^{n} \frac{1}{k_j! l_j!^{|k_j|}}.$$
To conclude, denote by $S^n_k$ the {\em Stirling numbers of the second kind}. By \citet[Corollary 2.9]{constantine1996multivariate} and  \citet{rennie1969stirling} we have
$$\nu! \sum_{1 \leq |\la| \leq n} \sum_{\tiny\begin{matrix}k_1\dots, k_n\\ l_1,\dots, l_n \end{matrix} \in p(\la, \nu)} ~\prod_{j=1}^{n} \frac{1}{k_j! l_j!^{|k_j|}} = \sum_{i=1}^n m^k S^n_k \leq m^n \sum_{i=1}^n \binom{n}{k} k^{n-k} \leq m^n (2n)^n.$$
\end{proof}

\begin{applemma}\label{lm:bound-composite-sobolev}
Let $\Psi_j:U_j \to B^k_{r_j}$ such that there exists $Q>0$ for which $\|D^\alpha \Psi_j^{-1}\|_{L^\infty(B^k_{r_j})}\leq Q^{|\alpha|}$ for $\alpha \in \N^k$. Then for any $q \geq k$, we have
$$
\|f \circ \Psi_j^{-1}\|_{W^q_2(B^k_{r_j})} \leq C_{d,k,R,r_j} q^k (2q d Q)^{q}\| f\|_{W^{q + (d+1)/2}_2(B^d_R)}.
$$ 
\end{applemma}
\begin{proof}
First note that $\|\cdot\|_{L^\infty(B^d_R)} \leq C_{d,R}\|\cdot\|_{W^{(d+1)/2}_2(B^d_R)}$ \citep{adams2003sobolev} for a constant $C_{d,R}$ depending only on $d$ and $R$. Therefore
\eqals{
\|(D^\alpha f) \circ \Psi_j^{-1}\|_{L^2(B^k_{r_j})} &\leq \textrm{vol}(B^k_{r_j})^{1/2}\|(D^\alpha f) \circ \Psi_j^{-1}\|_{L^\infty(B^k_{r_j})} \\
& = \textrm{vol}(B^k_{r_j})^{1/2}\|D^\alpha f\|_{L^\infty(U_j)} \leq \textrm{vol}(B^k_{r_j})^{1/2}\|D^\alpha f\|_{L^\infty(B^d_R)} \\
& \leq C_{d,R}\textrm{vol}(B^k_{r_j})^{1/2}\|D^\alpha f\|_{W^{(d+1)/2}_2(B^d_R)}.
}
Moreover note that $\|D^\alpha f\|_{W^{(d+1)/2}_2(B^d_R)} \leq \| f\|_{W^{|\alpha| + (d+1)/2}_2(B^d_R)}$.
By \cref{lm:bound-composition} we have that 
$$\|D^\alpha (f \circ \Psi_j^{-1})\|_{L^2(B^k_{r_j})} \leq (2|\alpha| d Q)^{|\alpha|} \max_{|\la| \leq |\alpha|} \|(D^\la f) \circ \Psi_j^{-1}\|_{L^2(B^k_{r_j})}.$$
By definition of Sobolev space $W^q_2(B^k_{r_j})$, we have
\eqals{
\|f \circ \Psi_j^{-1}\|_{W^q_2(B^k_{r_j})} &\leq \sum_{|\alpha| \leq q} \|D^\alpha (f \circ \Psi_j^{-1})\|_{L^2(B^k_{r_j})} \\
&\leq \sum_{|\alpha| \leq q} (2|\alpha| d Q)^{|\alpha|} \max_{|\la| \leq |\alpha|} \|(D^\lambda f) \circ \Psi_j^{-1}\|_{L^2(B^k_{r_j})}\\
&\leq C_q \max_{|\la| \leq q} \|(D^\lambda f) \circ \Psi_j^{-1}\|^2_{L^2(B^k_{r_j})},
}
where $C_q := \sum_{|\alpha| \leq q} (2|\alpha| d Q)^{|\alpha|}$.
Then,
$$
\|f \circ \Psi_j^{-1}\|^2_{W^q_2(B^k_{r_j})} \leq C_q C_{d,R} \textrm{vol}(B^k_{r_j})^{1/2}\| f\|_{W^{q + (d+1)/2}_2(B^d_R)}.
$$
The final result is obtained via the bound $C_q \leq (2q d Q)^{q}\binom{k+q}{k} \leq (2ek)^k (2q d Q)^{q} q^k$ for $q \geq k$.
\end{proof}

\begin{applemma}[Bounds for the incomplete gamma function]\label{lm:incomplete-gamma}
Denote by $\Gamma(a,x)$ the function defined as
$$\Gamma(a,x) = \int_{x}^\infty z^{a-1} e^{-z} dz,$$
for $a \in \R$ and $x > 0$. When $x \geq (a-1)_+ ~\wedge~ x > 0$, the following holds
$$\Gamma(a,x) \leq \frac{x}{x - (a-1)_+} x^{a-1} e^{-x},$$ 
 In particular $\Gamma(a,x) \leq 2 x^{a-1} e^{-x},$ for $x \geq 2 (a-1)_+ ~\wedge~ x  > 0$.
\end{applemma}
\begin{proof}
	Assume $x > 0$.
	When $a \leq 1$, the function $z^{a-1} e^{-z}$ is decreasing and in particular $z^{a-1} e^{-z} \leq x^{a-1} e^{-z}$ for $z \geq x$, so when $z \geq x$ we have
	$$\Gamma(a,x) = \int_{x}^\infty z^{a-1} e^{-z} dz \leq x^{a-1} \int_{x}^\infty  e^{-z} dz = x^{a-1}e^{-x}.$$
	When $a > 1$, for any $\tau \in (0,1)$, we have
	\eqals{
	\Gamma(a,x) &= \int_{x}^\infty z^{a-1} e^{-z} dz = \int_{x}^\infty (z^{a-1} e^{-\tau z})(e^{-(1-\tau)z}) dz \\
	& \leq  \sup_{z \geq x}(z^{a-1} e^{-\tau z}) ~ \int_{x}^\infty e^{-(1-\tau)z} dz = \frac{e^{-(1-\tau)x}}{1-\tau}\sup_{z \geq x}(z^{a-1} e^{-\tau z}).
}
Now note that the maximum of $z^{a-1} e^{-\tau z}$ is reached when $z = (a-1)/\tau$. When $x \geq a-1$, we can set $\tau = (a-1)/x$, so the maximum of $z^{a-1} e^{-\tau z}$ is exactly in $z = x$. In that case $\sup_{z \geq x}(z^{a-1} e^{-\tau z})  = x^{a-1} e^{-\tau x}$ and
\eqals{
	\Gamma(a,x) \leq \frac{x}{x - (a-1)} x^{a-1} e^{-x}.
}
The final result is obtained by gathering the cases $a \leq 1$ and $a > 0$ in the same expression.
\end{proof}

\begin{appcor}\label{cor:generalization-incomplete-gamma}
	Let $a \in \R$, $A, q_1, q_2, b > 0$.  When $q_2 A^b \geq 2\left(\frac{a+1}{b} - 1\right)_+$, the following holds
	$$ \int_A^\infty q_1 x^a e^{-q_2 x^b} dx \leq \frac{2q_1}{b q_2} A^{a+1 - 1/b} e^{-q_2 A^b}.$$
\end{appcor}
\begin{proof}
	By the change of variable $x = (u/q_2)^{1/b}$ we have
\eqals{
	\int_A^\infty q_1 x^a e^{-q_2 x^b} dx & = \frac{q_1 q_2^{-\frac{a+1}{b}}}{b}\int_{q_2 A^b}^\infty u^{\frac{a+1}{b}-1} e^{-u} du \\
	& = \frac{q_1 q_2^{-\frac{a+1}{b}}}{b} \Gamma\left(\frac{a+1}{b}, q_2 A^b\right) \\
	& \leq \frac{2q_1}{b q_2} A^{a+1 - 1/b} e^{-q_2 A^b}, 
}
where for the last step we used \cref{lm:incomplete-gamma}. 
\end{proof}

\section{Lipschitz properties of the Sinkhorn projection}\label{app:sink-stable}
We give a simple construction illustrating that the Sinkhorn projection operator is not Lipschitz in the standard sense. This stands in contrast with Proposition~\ref{prop:sink-perturbation-mat}, which illustrates that this projection is Lipschitz on the logarithmic scale.
\par This non-Lipschitz result holds even for the following simple rescaling of the $2 \times 2$ Birkhoff polyope:
\[
\cM := \left\{ P \in \R_{\geq 0}^{2 \times 2} \; : \; P1 = P^T1 = \begin{bmatrix}
\half  \\ \half
\end{bmatrix}  \right\}.
\]

\begin{prop}\label{prop:sink-not-stable}
The Sinkhorn projection operator onto $\cM$ is not Lipschitz for any norm on $\RR^{2 \times 2}$.
\end{prop}
\begin{proof}
By the equivalence of finite-dimensional norms, it suffices to prove this for $\|\cdot\|_1$,
for which we will show
\begin{align}
\sup_{K,K' \in \Delta_{2,2} \cap \R_{> 0}^{2 \times 2}}
\frac{\left\| \SinkcM(K) - \SinkcM(K') \right\|_1}{\|K - K'\|_1} = \infty.
\label{eq:prop-sink-not-stable}
\end{align}
The restriction to strictly positive matrices $\R_{> 0}^{2 \times 2}$ (rather than non-negative matrices) is to ensure that there is no issue of existence of Sinkhorn projections \citep[see, e.g.,][Section 2]{LSW98}.
\par For $\eps,\delta \in (0,1)$, define the matrix
\[
K_{\eps,\delta} := \begin{bmatrix}
1 - \eps & \eps \\
1 - \delta & \delta 
\end{bmatrix}
\]
and let $P_{\eps,\delta} := \SinkcM(K_{\eps,\delta})$ denote the Sinkhorn projection of $K_{\eps,\delta}$ onto $\cM$.
The polytope $\cM$ is parameterizable by a single scalar as follows:
\[
\cM = \left\{ 
M_a
\; : \; a \in [0,1/2]
\right\}\,, \quad \quad M_a := \begin{bmatrix}
a & \half - a \\
\half - a & a
\end{bmatrix}\,.
\]

By definition, $P_{\eps, \delta}$ is the unique matrix in $\cM$ of the form $D_1 K_{\eps, \delta} D_2$ for positive diagonal matrices $D_1$ and $D_2$.
Taking
\begin{equation*}
D_1 = \begin{bmatrix}
\sqrt{\tfrac{\delta}{\eps}} & 0 \\
0 & \sqrt{\tfrac{1 - \eps}{1 - \delta}}
\end{bmatrix}\,, \quad \quad 
D_2 = \frac{1}{\beta_{\eps, \delta}}\begin{bmatrix}
\sqrt{\tfrac{\eps}{1- \eps}} & 0 \\
0 & \sqrt{\tfrac{1 - \delta}{\delta}}
\end{bmatrix}
\end{equation*}
for $\beta_{\eps, \delta} = 2(\sqrt{\delta(1-\eps)} + \sqrt{\eps(1-\delta)})$, we verify
$D_1 K_{\eps, \delta} D_2 = M_{a_{\eps, \delta}}$, where $a_{\eps, \delta} :=
\frac{\sqrt{\delta(1 - \eps)}}{\beta_{\eps, \delta}}$.
Therefore $P_{\eps, \delta} = M_{a_{\eps, \delta}}$ for $\eps, \delta \in (0, 1)$.
\par Now parameterize $\eps := c\delta$ for some fixed constant $c \in (0, \infty)$ and consider taking $\delta \to 0^+$. Then $a_{c\delta, \delta} = \frac{\sqrt{1 - c\delta}}{2\left[\sqrt{c(1 - \delta)} + \sqrt{1 - c\delta} \right]}$, which for fixed $c$ becomes arbitrarily close to 
\[
g(c) := \frac{1}{2(\sqrt{c} + 1)}
\]
as $\delta$ approaches $0$. Thus $\|P_{c\delta, \delta} - M_{g(c)}\|_1 = o_{\delta}(1)$ and similarly $\|P_{\delta/c, \delta} - M_{g(1/c)}\|_1 = o_{\delta}(1)$. We therefore conclude that for any constant $c \in (0, \infty) \setminus \{1\}$, although
\begin{align*}
\|K_{c\delta,\delta} - K_{\delta/c,\delta}\|_1
=
2\delta \abs{c - 1/c}
= o_{\delta}(1),
\end{align*}
vanishes as $\delta \to 0^+$, the quantity
\begin{align*}
\left\|\SinkcM(K_{c\delta, \delta}) - \SinkcM(K_{\delta/c,\delta}) \right\|_1
&=
\left\|
P_{c\delta,c} - P_{\delta/c,c}
\right\|_1
\\ &=
\left\|
M_{g(c)} - M_{g(1/c)}
\right\|_1
+ o_{\delta}(1)
\\ &=
4\abs{g(c) - g(1/c)}
+ o_{\delta}(1)
\end{align*}
does not vanish. Therefore combining the above two displays and taking, e.g., $c = 2$ proves~\eqref{eq:prop-sink-not-stable}.
\end{proof}

\addcontentsline{toc}{section}{References}
\bibliographystyle{abbrvnat}
\bibliography{nystrom_sinkhorn}
\end{document}